\documentclass[journal,twoside,web]{ieeecolor}
\usepackage{generic}
\usepackage{cite}
\usepackage{amsmath,amssymb,amsfonts}
\usepackage{algorithmic}
\usepackage{graphicx}
\usepackage{textcomp}

\PassOptionsToPackage{numbers, compress}{natbib}

\usepackage{amsmath,stmaryrd,amssymb} 
\usepackage{enumerate} 

\usepackage{multirow}

\newtheorem{thm}{Theorem}[section]

\newtheorem{asmp}[thm]{Assumption}

\newtheorem{rem}[thm]{Remark}

\usepackage{hyperref,cleveref}

\newtheorem{theorem}{Theorem}[section]

\newtheorem{lemma}[theorem]{Lemma}

\usepackage[linesnumbered,ruled,vlined]{algorithm2e}

\usepackage{hyperref}

\usepackage{graphicx}
\usepackage{wrapfig}
\usepackage{subfigure}

\usepackage{footmisc}

\newcommand{\diag}{\mathtt{diag}}
\newcommand{\Exp}{\mathbb{E}}

\def\BibTeX{{\rm B\kern-.05em{\sc i\kern-.025em b}\kern-.08em
    T\kern-.1667em\lower.7ex\hbox{E}\kern-.125emX}}
\begin{document}
\title{Cooperative Multi-Agent Reinforcement Learning with Partial Observations}
\author{Yan Zhang, \IEEEmembership{Student Member, IEEE}, Michael M. Zavlanos, \IEEEmembership{Senior Member, IEEE}
\thanks{This work is supported in part by AFOSR under award \#FA9550-19-1-0169 and by NSF under award CNS-1932011.}
\thanks{Yan Zhang and Michael M. Zavlanos are with Mechanical Engineering and Material Science Department at Duke University, Durham, NC 27708 USA (e-mail: yan.zhang2@duke.edu, michael.zavlanos@duke.edu). }
}

\maketitle

\begin{abstract}
In this paper, we propose a distributed zeroth-order policy optimization method for Multi-Agent Reinforcement Learning (MARL). Existing MARL algorithms often assume that every agent can observe the states and actions of all the other agents in the network. This can be impractical in large-scale problems, where sharing the state and action information with multi-hop neighbors may incur significant communication overhead. The advantage of the proposed zeroth-order policy optimization method is that it allows the agents to compute the local policy gradients needed to update their local policy functions using local estimates of the global accumulated rewards that depend on partial state and action information only and can be obtained using consensus. Specifically, to calculate the local policy gradients, we develop a new distributed zeroth-order policy gradient estimator that relies on one-point residual-feedback which, compared to existing zeroth-order estimators that also rely on one-point feedback, significantly reduces the variance of the policy gradient estimates improving, in this way, the learning performance.
We show that the proposed distributed zeroth-order policy optimization method with constant stepsize converges to the neighborhood of a policy that is a stationary point of the global objective function. 
The size of this neighborhood depends on the agents' learning rates, the exploration parameters, and the number of consensus steps used to calculate the local estimates of the global accumulated rewards.
Moreover, we provide numerical experiments that demonstrate that our new zeroth-order policy gradient estimator is more sample-efficient compared to other existing one-point estimators.
\end{abstract}

\begin{IEEEkeywords}
Distributed Zeroth-Order Optimization, Multi-Agent Reinforcement Learning, Partial Observation
\end{IEEEkeywords}

\section{Introduction}
Multi-Agent Reinforcement Learning (MARL) has received a lot of attention in recent years due to its wide applicability in real-world large-scale decision making problems, e.g., cloud autonomous driving, distributed multi-robot planning, and distributed resource allocation, to name a few. The goal is to enable a team of agents to collaboratively determine the global optimal policy that maximizes the sum of their local accumulated rewards. To do so, the agents typically need to communicate with each other in order to obtain information about the global state and action of the team. This is because their states and rewards are generally affected by the actions of their other peers.
However, sharing such information can be undesirable, due to significant communication overhead or privacy concerns. Therefore, there is a great need for MARL algorithms that rely only on partial observations of the global state and action information.

A major challenge in developing cooperative MARL methods under partial observations is that the environment, as perceived by every individual agent when it interacts with the other agents, is non-stationary since it changes as a result of changes in the policies of those agents \cite{gupta2017cooperative}.
In \cite{gupta2017cooperative,lowe2017multi,foerster2017counterfactual}, this challenge is addressed using a centralized Critic function that can mitigate the effect of non-stationarity in the learning process. Then, the trained policies can be executed in a decentralized way. In \cite{omidshafiei2017deep}, a distributed offline experience replay technique is developed to enable fully decentralized training, which requires that all agents receive a global reward at each timestep. However, when the global reward is defined as the sum of local agent rewards, as in cooperative MARL, this global reward can not be easily available to the local agents in practice.
Cooperative MARL methods that maximize the sum of local rewards are considered in \cite{zhang2018fullyICML,ZhangCDC2019_DistributedRL,suttle2019multi,heredia2019distributed}. These works develop fully decentralized Actor-Critic methods where the agents maintain local estimates of the global value or policy functions, that depend on the states and actions of all other agents and update those estimates until they reach consensus. Then, these local estimates of the global value or policy functions are used to compute the policy gradient estimates needed for optimization. Since these policy gradient estimates require knowledge of the global state and action information, such Actor-Critic methods can not be used for cooperative MARL with partial state and action information.

In this paper, we propose a new distributed zeroth-order estimator of the local policy gradients, which is an extension of the one-point zeroth-order gradient estimator developed in \cite{zhang2020improving} for centralized optimization problems that estimates the gradient using the residual of the function values at two consecutive iterations. As such, it queries the function value only once at each iteration. Specifically, our proposed estimator computes the local policy gradients by locally perturbing the local agent policies, using information about the sum of local accumulated rewards that can be obtained using consensus; the sum of local accumulated rewards is global information that is not otherwise accessible to the local agents.The advantage of our proposed zeroth-order policy optimization method is that it makes it possible to estimate the local policy gradients without having access to the global Critic function, which is not available under partial state and action information.

{\bf Related work:} Zeroth-order policy optimization has been considered in \cite{fazel2018global,malik2018derivative} for a special case of single-agent RL problems, namely, Linear Quadratic Regulation (LQR) problems. These results were extended in \cite{li2019distributed} for distributed LQR problems. All these works use the one-point zeroth-order policy gradient estimator proposed in \cite{flaxman2005online,nesterov2013introductory}, which is known to have large variance that slows down learning \cite{duchi2015optimal}.
Instead, our proposed distributed residual-feedback estimator returns policy gradient estimates with significantly lower variance which improves the learning performance. 
Our work is also related to distributed optimization methods, as in \cite{hajinezhad2018gradient,sahu2018distributed,dixit2020online,bastianello2021distributed,yuan2021distributed}. In these works, the agents collaborate to find a decision variable that maximizes the sum of local objective functions. For this, the method in  \cite{dixit2020online} assumes that the gradient can be computed, while the methods proposed in  \cite{hajinezhad2018gradient,sahu2018distributed,bastianello2021distributed,yuan2021distributed} do not assume knowledge of the gradient and instead compute the zeroth-order gradient estimates of the local objective functions with respect to all decision variables, even those owned by other agents.
However, in MARL with partial observations, the agents can not have access to all decision variables and, therefore, can only compute the zeroth-order gradient estimate of their local objective functions with respect to those partial decision variables to which they have access. Therefore, the methods in \cite{hajinezhad2018gradient,sahu2018distributed} cannot be used to solve the MARL problems considered here. 
In fact, recent work on cooperative MARL \cite{yang2020q} relies on the assumption that the local agents can observe the team reward. This is in contrast to the problem considered here where the agents can only observe their local rewards, while the team rewards, which are the summation of local rewards, are not accessible locally. Therefore, the method in \cite{yang2020q} cannot be applied to the problem considered in this paper.
Related is finally work on multi-agent game formulations of MARL problems \cite{lanctot2017unified,srinivasan2018actor}. While these problems rely on partial state and action information, they are non-cooperative in nature since the goal of the agents is to optimize their local policies in order to maximize their local accumulated rewards, as opposed to maximizing the sum of their local accumulated rewards. As a result, in these problems, the agents converge to a Nash equilibrium point rather than an optimal policy that maximizes the global accumulated rewards.

{\bf Contributions:} In this paper, we propose a new distributed zeroth-order policy optimization method for general cooperative MARL problems.
Compared to the one-point policy gradient estimators in \cite{fazel2018global,malik2018derivative,li2019distributed}, our proposed residual-feedback policy gradient estimator reduces the variance of the policy gradient estimates and, therefore, improves the learning performance. Compared to the centralized estimator in \cite{zhang2020improving} that produces unbiased gradient estimates, the proposed distributed policy gradient estimator is biased due to possible consensus errors in distributedly estimating the sum of local accumulated rewards needed for the estimation of the policy gradients.  
We show that the proposed zeroth-order policy optimization method with constant stepsize converges to a neighborhood of a stationary point (policy) of the global objective function. The size of this neighborhood depends on the number of consensus steps needed to control the bias in the policy gradient estimates.
Moreover, we propose a value tracking method to reduce the numer of consensus steps needed to achieve a desired user-specified solution accuracy.
Finally, compared to existing distributed zeroth-order optimization methods \cite{hajinezhad2018gradient,sahu2018distributed}, our method is the first that does not require a perturbation of all decision variables or knowledge of all local objective function values to compute the zeroth-order gradients of the local objective functions, but still converges to a stationary point of the global objective function. To the best of our knowledge, this is the first work that provides convergence guarantees for collaborative multi-agent reinforcement learinng problems when global states and actions are partially observable and agents can only receive local rewards.

The rest of the paper is organized as follows. In Section 2, we present the MARL problem under consideration and introduce preliminary results on zeroth-order optimization. In Section 3, we develop our proposed algorithm and present the convergence analysis. In Section 4, we verify the effectiveness of our algorithms using numerical experiments. In Section 5, we conclude the paper.
\vspace{-2mm}

\section{Preliminaries and Problem Formulation}
\label{sec:MARL}
Consider a multi-agent system consisting of $N$ agents. The agent dynamics are governed by a Makov Decision Process (MDP) defined by a tuple $(\mathcal{S}, \mathcal{A}, \mathcal{R}, \mathcal{P}, \gamma)$, where $s_t = [s_{1,t}, s_{2,t}, \dots, s_{N,t}] \in \mathcal{S}$ and $a_t = [a_{1,t}, a_{2,t}, \dots, a_{N,t}] \in \mathcal{A}$ denote the joint state and action spaces of the $N$ agents at time instant $t$. The reward vector $r = [r_{1,t}, r_{2,t}, \dots, r_{N,t}] \in \mathcal{R}$ denotes the local rewards received by each agent at time $t$. The local reward $r_{i, t}(s_t, a_t, w_t)$ is affected by the joint state and action of all the agents in the network, and is also subject to noise $w_t$. The transition function $P(s_t, a_t, s_{t+1}) : \mathcal{S} \times \mathcal{A} \times \mathcal{S} \rightarrow [0, 1] \in \mathcal{P}$ denotes the probability of transitioning to state $s_{t+1}$ when the agents take action $a_t$ at state $s_t$. Let $o_{i,t} \in \mathcal{O}_i$ represent the local observation received at agent $i$ at time $t$, which contains partial entries of the joint state and action vectors, $s_t$ and $a_t$. Agent $i$ selects its action $a_{i,t}$ based on the observation $o_{i,t}$ using its local policy function $\pi_i: \mathcal{O}_i \rightarrow \mathcal{A}_i$. Let $\pi$ denote the joint policy function which consists of all local policy functions $\pi_i$. Then, the accumulated discounted reward received by agent $i$ is defined as $Q^\pi_i(s, a) = \mathbb{E}[ \sum_{t=0}^{T} \gamma^{t} r_{i,t} | s_0 = s, a_0 = a]$ or $V^\pi_i(s) = \mathbb{E}[ \sum_{t=0}^{T} \gamma^{t} r_{i,t} | s_0 = s]$,
when the agents start from the state-action pair $(s, a)$ or state $s$, follow the joint policy $\pi$, and apply a discount factor $\gamma \leq 1$ to their future rewards \footnote{Although we consider a task with accumulated discounted rewards in this paper, our proposed methods can be easily adapted to task considering averaged rewards.}. 

Our goal in this paper is to find an optimal joint policy function $\pi^\ast$ that solves the problem $\max_{\pi} \; \frac{1}{N}\sum_{i = 1}^{N} J_i(\pi)$,
where $J_i(\pi) = \mathbb{E}_{(s_0,a_0) \sim \rho_0} [ Q_i^\pi(s_0, a_0) ]$ and $\rho_0$ is a distribution that the initial state-action pair is sampled from. 
To do so, we assume that the local policy function $\pi_i$ is parameterized as $\pi_i(\theta_{i,k})$, where $\theta_{i,k} \in \mathbb{R}^{d_i}$ is the local policy parameter during episode $k$. Stacking these local policy parameters into the global policy parameter vector $\theta \in \mathbb{R}^d$, we can rewrite the problem we consider in this paper as
\begin{equation}
\label{eqn:ProbForm_theta}
\max_{\theta \in \mathbb{R}^d }  \; J(\theta) := \frac{1}{N}\sum_{i = 1}^{N} J_i(\theta).
\end{equation}

\begin{asmp}
	\label{asmp:Obj}
	We assume that the local objective function $J_i(\theta)$ is non-convex and non-smooth for all $i = 1, 2,  \dots, N$.
\end{asmp}

Problem~\eqref{eqn:ProbForm_theta} can be solved using distributed Actor-Critic methods as in \cite{zhang2018fullyICML,ZhangCDC2019_DistributedRL}. 
These methods require that all agents maintain local estimates of the global value function or the global policy function and that these local estimates are parameterized in the same way and depend on the global states and actions of all other agents.
Therefore, they cannot be used for MARL with partial state and action information. Instead, in this paper, we propose a new distributed zeroth-order policy optimization method that relies on the stochastic gradient ascent update
\begin{align}
	\label{eqn:StochasticGradient}
	\theta_{k+1} = \theta_k + \alpha \nabla J(\theta_k) + \epsilon_k
\end{align}
to determine optimal policy parameters $\theta_k$ that solve Problem~\eqref{eqn:ProbForm_theta}, where $\epsilon_k$ represents the noise in gradient estimate and $\alpha$ denotes the stepsize. The key idea that makes it possible to use partial state and action information in the update~\eqref{eqn:StochasticGradient} is the use zeroth-order estimators $\tilde{\nabla} J(\theta_k)$ of the true policy gradient $\nabla J(\theta_k)$.
Zeroth-order gradient estimators have been recently proposed in \cite{duchi2015optimal,nesterov2017random,gasnikov2017stochastic}, and take the form
\begin{align}
	\tilde{\nabla} J(\theta_k) & = \frac{J(\theta_k  + \delta u_k, \xi_k)}{\delta} u_k, & \label{eqn:OnePoint}\\
	\text{or } \; \tilde{\nabla} J(\theta_k) & = \frac{J(\theta_k  + \delta u_k, \xi_k) - J(\theta_k  - \delta u_k, \xi_k)}{2\delta} u_k, & \label{eqn:TwoPoint}
\end{align}
where $J(\theta_k  + \delta u_k, \xi_k)$ is an unbiased noisy sample 
\footnote{ For RL problems, the noise vector $\xi_k$ in the function evaluation is due to noise in the initial state and action samples, the state transition dynamics and the reward signals.}
of the accumulated rewards under the perturbed policy $\theta_k  + \delta u_k$, $\delta \in \mathbb{R}$ is the exploration parameter, $u_k \in \mathbb{R}^d$ is the random exploration direction sampled independently from a standard multivariate normal distribution $\mathcal{N}(0, I_d)$ and $I_d$ is an identity matrix with dimension $d$. 
The estimator~\eqref{eqn:OnePoint} is called a {\it one-point estimator} because it requires one policy evaluation at a single perturbed policy, $\theta_k  + \delta u_k$. On the other hand, the estimator~\eqref{eqn:TwoPoint} is called {\it a two-point estimator} because it requires two policy evaluations at two different perturbed policies, $\theta_k  + \delta u_k$ and $\theta_k  - \delta u_k$. 
While the two-point estimator~\eqref{eqn:TwoPoint} typically produces gradient estimates with lower variance, it is difficult  to use in decentralized MARL problems.
 This is because, to compute zeroth-order policy gradient estimates in MARL, requires coordination between the agents to evaluate their local policies. Specifically, the agents need to initialize an episode, randomly perturb their local policies and implement the perturbed policies until the end of episode. This procedure requires synchronization, which can incur delays. And the more points are used to estimate the zeroth-order gradient, the longer these delays become as multi-point gradient estimation requires synchronization over multiple episodes to implement one udpate. In practice, this waiting time forces the agents to remain at a sub-optimal policy for longer than the proposed one-point gradient estimator.

To address the above limitations akin to estimators~\eqref{eqn:OnePoint} and \eqref{eqn:TwoPoint}, in this paper, we adopt the one-point residual-feedback policy gradient estimator
\begin{align}
	\tilde{\nabla} J(\theta_k) = \frac{J(\theta_k  + \delta u_k, \xi_k) - J(\theta_{k-1}  + \delta u_{k-1}, \xi_{k-1})}{\delta} u_k \label{eqn:Residual}
\end{align}
originally proposed in \cite{zhang2020improving}. Same as the estimator \eqref{eqn:OnePoint}, the estimator \eqref{eqn:Residual} only requires one policy evaluation at each iteration, but can use the history of policy evaluations to effectively reduce the variance of the current policy gradient estimate and, therefore, improve the learning rate.  
We note that the estimator \eqref{eqn:Residual} cannot be directly used to solve MARL problems where the agents can only observe the value of their local objective functions $J_i$. To address this challenge, in this paper, we let agents implement a finite number of consensus  steps to approximate the value of the global objective function $J$ and analyze the effect of consensus errors on the gradient estimator~\eqref{eqn:Residual}, as we discuss below. \footnote{We note that since the original submission of this paper, a new one-point zeroth-order estimator has been proposed in \cite{chen2022improve} that improves on the performance of the estimator in~\eqref{eqn:Residual}. However, this method is centralized and does not directly apply to the MARL problem under consideration.}

According to \cite{nesterov2017random, zhang2020improving}, both estimators~\eqref{eqn:OnePoint} and \eqref{eqn:Residual} provide unbiased gradient estimates of a smoothed function $J_\delta(\theta)$ at $\theta_k$, where $J_\delta(\theta)$ is defined as $J_\delta(\theta) := \mathbb{E}_u\big[ J(\theta + \delta u)\big]$ and $u$ is subject to a standard multivariate normal distribution. Therefore, updating the policy parameter $\theta_k$ as in \eqref{eqn:StochasticGradient} using the gradient estimates~\eqref{eqn:OnePoint} or \eqref{eqn:Residual} will in fact converge to a stationary point of the smoothed function $J_{\delta}(\theta)$ rather than a stationary point of the value function $J(\theta)$ that may be nonsmooth. To ensure that the stationary point found by this process is meaningful for the original MARL problem, we need to define appropriate optimality conditions that additionally ensure that $J_{\delta}(\theta)$ and $J(\theta)$ are close to each other. 
Specifically, we consider the following optimality criterion
\begin{equation}
\label{eqn:Criterion}
\|\nabla J_\delta(\theta)\|^2 \leq \epsilon, \; \text{ and } |J_\delta(\theta) - J(\theta)| \leq \epsilon_J,
\end{equation}
which suggests that $\theta$ is an $\epsilon-$stationary point of the smoothed value function $J_\delta(\theta)$, and that the smoothed value function $J_\delta(\theta)$ is $\epsilon_J$-close to the true value function $J(\theta)$. 
In this paper, we use the notation $\|\cdot\|$ to denote the Euclidean vector norm, or the induced matrix norm induced by the Euclidean norm.
To bound the distance between the smoothed function $J_{\delta}(\theta)$ and the original value funtion $J(\theta)$ in \eqref{eqn:Criterion}, we need the following assumption on the value function $J(\theta)$.
\begin{asmp}
	\label{asm:Lipschitz}
	The function $J(\theta)$ is Lipschitz with constant $L_0$, that is, $|J(\theta_1) - J(\theta_2)| \leq L_0 \|\theta_1 - \theta_2\|$, for all $\theta_1, \theta_2 \in \mathbb{R}^d$. 
\end{asmp}
We note that Assumption~\ref{asm:Lipschitz} can be restrictive in selecting the class of policy functions. However, even with Assumption~\ref{asm:Lipschitz}, to the best of our knowledge, this work is the first to show convergence of distributed policy optimization under partial observability. Furthermore, it is the first to show convergence of distributed zeroth-order optimization when the local agents can only perturb their local decision variables and observe their local objective function value. For comparison, the methods in \cite{hajinezhad2018gradient,sahu2018distributed} require that the full decision variable is perturbed at every local agent, which imply that they can all observe global states and actions.
Given Assumption~\ref{asm:Lipschitz}, the following result for the smoothed value function $J_{\delta}(\theta)$ holds.
\begin{lemma}
	\label{lem:ApproxLipschitz}
	(Gaussian Approximation \cite{nesterov2017random}) Given Assumption~\ref{asm:Lipschitz}, the smoothed function $J_\delta(\theta)$ satisfies $|J_\delta(\theta) - J(\theta)| \leq \delta L_0 \sqrt{d}, \; \text{for all }\theta \in \mathbb{R}^d$.
\end{lemma}
According to Lemma \ref{lem:ApproxLipschitz}, to control the approximation accuracy of the smoothed function $J_\delta(\theta)$, the parameter $\delta$ needs to be selected appropriately. The choice of this parameter will be discussed in Section~\ref{sec:Algorithm}.
Note that although the random exploration direction $u_k \sim \mathcal{N}(0, I_d)$ needed to evaluate the estimator~\eqref{eqn:Residual} can be sampled in a fully decentralized way, the global value $J(\theta_k + \delta u_k) = \frac{1}{N}\sum_{i = 1}^{N} J_i(\theta_k  + \delta u_k, \xi_k)$ is not accessible by the local agents. In the next section, we design a new algorithm that relies on partial state and action information only to produce a fully decentralized implementation of the estimator~\eqref{eqn:Residual}. 


\section{Algorithm Design and Theoretical Analysis}
\label{sec:Algorithm}
In this section, we propose a fully distributed zeroth-orther policy optimization algorithm for MARL that employs the residual-feedback zeroth-order policy gradient estimator~\eqref{eqn:Residual}. 
Specifically, we first introduce a consensus step so that the global value $J(\theta_k + \delta u_k, \xi_k)$ in the estimator \eqref{eqn:Residual} can be computed locally. Given a finite number of consensus iterations, the local estimates of $J(\theta_k + \delta u_k, \xi_k)$ will be inexact and, therefore, the local policy gradient estimates will be biased. 
To control this bias, we then introduce a value tracking technique that reduces the bias at the current episode using the local estimates of $J(\theta_k+\delta u_k, \xi_k)$ from previous episodes.
Finally, we provide convergence results showing that the proposed distributed zeroth-order policy optimization method with constant stepsize converges to a neighborhood of the stationary point of the smoothed global objective function.The size of this neighborhood is controlled by the number of consensus steps during each episode. Proofs of all theoretical results that follow can be found in the Appendix.

Our proposed algorithm is summarized in Algorithm~\ref{alg:DZO}. 
In what follows, we also assume that the $N$ agents form a communication graph $\mathcal{G} = (\mathcal{V}, \mathcal{E})$, where $\mathcal{V} = \{1, 2, \dots, N\}$ is the index set of agents and $\mathcal{E}$ represents the set of edges. The edge $(i,j) \in \mathcal{E}$ if agents $i$ and $j \in \mathcal{N}$ can directly send information to each other. Moreover, we define by $W \in \mathbb{R}^{N \times N}$ a weight matrix associated with the graph $\mathcal{G}$ such that the entry $W_{ij} > 0$ when $(i,j) \in \mathcal{E}$ and $W_{ij} = 0$ otherwise. 
Note that the communication graph $\mathcal{G}$ is independent of the coupling between agents in their state transition function $\mathcal{P}(s_t, a_t, s_{t+1})$ and reward functions $\{r_i(s_t, a_t)\}$ defined in Section~\ref{sec:MARL}.

\begin{algorithm}[t]
	\small
	\SetAlgoLined
	\KwIn{Exploration parameter $\delta$, stepsize $\alpha$, consensus matrix $W$, number of consenseus steps $N_c$, initial policy parameter $\theta_0$, discount ratio $\gamma$, maximum number of time steps run per episode $t_{\max}$, number of episodes $K$, and the logic variable $\mathtt{DoTracking}$.}
	Set $\mu_i^{-1}(N_c) = 0$ for all $i = 1, 2, \dots, N$ \;
	\For{episode $k = 0, 1, 2, \dots, K$}{
		For agents $i = 1, 2, \dots, N$, let agent $i$ sample a random exploration direction $u_{i,k}$ from the standard multivariate normal distribution \;
		
		Let all agents implement their perturbed policy $\pi_i(\theta_{i,k} + \delta u_{i,k})$ for $t_{\max}$ time steps and construct unbiased estimates of their local accumulated rewards $\{ J_i(\theta_k + \delta u_k, \xi_k)\}$ \;
		
		For all agent $i = 1, 2, \dots, N$, \\
		\eIf{$\mathtt{DoTracking} = \mathtt{False} \; \mathbf{ or } \; k == 0$}
		{set $\mu_i^k(0) = J_i(\theta_k + \delta u_k, \xi_k)$ \;}{
			set $\mu_i^k(0) = \mu_i^{k-1}(N_c) + J_i(\theta_k + \delta u_k, \xi_k) - J_i(\theta_{k-1} + \delta u_{k-1}, \xi_{k-1})$ \;}
		
		\For{ $m = 0, 1, 2, \dots, N_c-1$}{
			For agents $i = 1, 2, \dots, N$, let agent $i$ send $\mu_i^k(m)$ to its direct neighbors $j \in \mathcal{N}_i$ and conduct local averaging by computing $\mu_i^k(m+1) = \sum_{j \in \mathcal{N}_i} W_{ij} \mu_j^k(m)$ \;
		}
		For agents $i = 1, 2, \dots, N$, let agent $i$ update its current policy parameter $\theta_{i, k}$ by
		\begin{align}
			\label{eqn:Alg_Residual}
			& \theta_{i,k+1} = \theta_{i,k} + \alpha  \frac{\mu_i^k(N_c) - \mu_i^{k-1}(N_c)}{\delta} u_{i,k}. & 
		\end{align} 
	}
\KwOut{ Uniformly sample an integer $k$ within the interval $[0, K]$ and output $\theta_k$. }
	\caption{Distributed Residual-Feedback Zeroth-Order Policy Optimization}
	\label{alg:DZO}
\end{algorithm}


\setlength{\textfloatsep}{8pt}

\subsection{Distributed Residual-Feedback Zeroth-Order Policy Optimization}
\label{sec:DZO}

In this section, we describe and analyze our proposed residual-feedback zeroth-order policy optimization algorithm in the absence of value tracking, i.e., when $\mathtt{DoTracking} = \mathtt{False}$ in Algorithm~\ref{alg:DZO} (line 6).
Specifically, at the beginning of episode $k$, the agents randomly perturb their current policy parameters $\theta_{k}$ using a random exploration direction $u_k$ and conduct on-policy local policy evaluation to obtain an unbiased estimate of the local accumulated rewards $\{ J_i(\theta_k + \delta u_k) \}_{i = 1 ,2, \dots, N}$ (lines 3-4). To conduct local policy evaluations, existing MARL methods \cite{zhang2018fullyICML,ZhangCDC2019_DistributedRL} usually assume that the global state-action pairs $(s_t, a_t)$ are available to all local agents. Under this assumption, it is possible to update the local Critic functions $Q_i^\pi(s_t, a_t)$ in \cite{zhang2018fullyICML,ZhangCDC2019_DistributedRL} to reduce the variance of policy evaluations and, therefore, the variance of the policy gradient estimates \cite{konda2000actor}. However, when the agents only have access to local observations $o_{i,t}$ which contain partial entries of $(s_t, a_t)$, these methods cannot be used. Therefore, in this paper, evaluate the local policies as $J_i(\theta_k + \delta u_k, \xi_k) = r_i(1) + \gamma r_i(2) + \gamma^2 r_i(3) + \dots + \gamma^{t_{max} - 1} r_i(t_{max})$,
same as in REINFORCE \cite{williams1992simple}. This policy evaluation method can be implemented in a fully decentralized way but is subject to large variance, which increases the variance of the zeroth-order policy gradient estimates and degrades the converegence speed of the algorithm. 
The residual-feedback policy gradient estimator~\eqref{eqn:Residual} can effectively reduce this variance as we will discuss later.
In what follows, we make the following assumption on the local policy value estimator.

\begin{asmp}
	\label{asm:UnbiasedEval}
	For all agents, the local policy evaluation is subject to bounded variance. That is, $\mathbb{E}\big[ ( J_i(\theta, \xi) - J_i(\theta) )^2 \big] \leq \sigma^2$ for $i = 1, 2, \dots, N$.
\end{asmp}
After all agents compute  local policy values $J_i(\theta_{i, k} + \delta u_k, \xi_k)$ as $J_i(\theta_k + \delta u_k, \xi_k) = r_i(1) + \gamma r_i(2) + \gamma^2 r_i(3) + \dots + \gamma^{t_{max} - 1} r_i(t_{max})$, which is an unbiased estimation according to \cite{williams1992simple}, they conduct $N_c$ rounds of local averaging on their local policy values $\{ J_i(\theta_k + \delta u_k, \xi_k) \}_{i = 1 ,2, \dots, N}$ (lines 7, 11-13). As a result, they obtain inexact estimates $\mu_i^k(N_c+1)$ of the global accumulated rewards $J(\theta_k + \delta u_k, \xi_k)$. To bound this estimation error, we need the following two assumptions.
\begin{asmp}
	\label{asm:Graph}
	The undirected communication graph $\mathcal{G}$ is connected and fixed for all episodes. In addition, the associated weight matrix $W$ is doubly stochastic. That is, $W \mathbf{1}_N = \mathbf{1}_N $ and $W^T \mathbf{1}  = \mathbf{1}_N$, where $\mathbf{1}_N$ denotes a $N$ dimensional vector of elements $1$.
\end{asmp}

For simplicity of notations, in the following we neglect the subscript $N$ in $\mathbf{1}_N$ and write it as $\mathbf{1}$. The dimension of the vector can be taken as appropriate within its context.
Assumption~\ref{asm:Graph} is introduced to ensure that agents' local estimates $\mu_i^k$ of the global objective function value can reach consensus on their average $\frac{1}{N} \sum_i {\mu_i^k}$ within an error that depends on the number of consensus steps. In this paper we assume that the communication networks is fixed. If it is not, then consensus methods for time-varying graphs, such as that proposed in \cite{nedic2017achieving} can be used to estimate the global objective function value. The analysis in this case is left for future research
\begin{asmp}
	\label{asm:BoundVal}
	The local values $J_i(\theta, \xi)$ are upper bounded by $J_u$ and lower bounded by $J_l$ for all $i = 1, 2, \dots, N$ and all policy parameters $\theta$.
\end{asmp}

Assumption~\ref{asm:BoundVal} can be easily satisfied for episodic RL problems with bounded rewards over all state-action pairs.
Let $\vec{\mu}^k(m) = [\mu_1^k(m), \dots, \mu_N^k(m)]^T$. Then, we can show the following lemma.
\begin{lemma}
	\label{lem:ConsensusErr_1}
	Given Assumptions~\ref{asm:Graph} and \ref{asm:BoundVal}, we have that $\|\vec{\mu}^k(N_c) - J(\theta_{k} + \delta u_k, \xi_k) \mathbf{1}\| \leq \rho_W^{N_c} \sqrt{N} (J_u - J_l)$, where $\rho_W = \|W - \frac{1}{N}\mathbf{1}\mathbf{1}^T\| < 1$. 
\end{lemma}

Lemma~\ref{lem:ConsensusErr_1} shows that the bias in the local estimate $\vec{\mu}^k(N_c)$ can be controlled by choosing a large enough $N_c$, when the local policy values are upper and lower bounded by $J_u$ and $J_l$. Using the estimates $\vec{\mu}^k(N_c)$, the agents can then construct the decentralized policy gradients~\eqref{eqn:Alg_Residual} and update their policy parameters $\theta_{i,k}$ (line 14). This completes episode $k$.
The decentralized residual-feedback estimator \eqref{eqn:Alg_Residual} can reduce the variance of the policy gradient estimates, since the value estimate of the last policy iterate $\vec{\mu}_i^{k-1}(N_c)$ can provide a baseline to compare $\vec{\mu}_i^k(N_c)$ to.
Effectively, the value estimate of the last policy iterate has an analogous variance reduction effect to the state value $V^\pi(s)$ that is used as a baseline for the action value $Q^\pi(s, a)$ in Actor-Critic methods \cite{konda2000actor}. 
\footnote[3]{Note that the fact that the value of past policy iterates can also be used as a baseline to reduce the variance of policy gradient estimates may be of interest in its own right in the development of policy gradient methods.}
Next, we show how to select $N_c$ so that the optimality criterion~\eqref{eqn:Criterion} is satisfied.
\begin{theorem}
	\label{thm:DZO}
	{\bf (Learning Rate of Algorithm~\ref{alg:DZO} without Value Tracking)} Let Assumptions~\ref{asm:Lipschitz}, \ref{asm:UnbiasedEval}, \ref{asm:Graph} and \ref{asm:BoundVal} hold and define $\delta = \frac{\epsilon_J}{\sqrt{d}L_0}$, $  \alpha = \frac{\epsilon_J^{1.5}} {4d^{1.5}L_0^2\sqrt{K}}$, and $N_c \geq \log( \frac{\sqrt{\epsilon} \epsilon_J} { \sqrt{2} d^{1.5} L_0 (J_u - J_l))} ) / \log(\rho_W)$.
	Then, running Algorithm~\ref{alg:DZO} with $\mathtt{DoTracking} = \mathtt{False}$, we have that $\frac{1}{K} \sum_{k = 0}^{K-1} \mathbb{E}[ \| \nabla J_\delta(\theta_k) \|^2 ] \leq \mathcal{O}( d^{1.5} \epsilon_J^{-1.5} K^{-0.5} ) + \frac{\epsilon}{2}$, where the expectation is taken over the trajectory of sampled vector $u_k$ and evaluation noise $\xi_k$.
\end{theorem}

As shown in Theorem~\ref{thm:DZO}, Algorithm~\ref{alg:DZO} converges to a neighborhood of the stationary point of the smoothed global objective function. 
Specifically, according to the optimality conditions in~\eqref{eqn:Criterion}, the approximation error on the smoothed objective function $J_\delta(\theta_k)$ compared to the original function $J(\theta_k)$ is controled by the parameter $\epsilon_J$. And the neighborhood around the stationary solution is characterized by the parameter $\epsilon$. Given the user specified parameters $\epsilon_J$ and $\epsilon$, Theorem~\ref{thm:DZO} says that the number of iterations $K$ can be selected according to the bound in \eqref{eqn:Thm3.2_12}. Note that computing this bound requires knowledge of the problem parameters $J_\delta^\ast - J_\delta(\theta_{0})$, $ \Exp[\|g_\delta(\theta_0)\|^2]$ and $L_0$. In practice, we can replace these parameters in \eqref{eqn:Thm3.2_12} by bounds selected sufficiently large. Note also that selecting the number $K$ according to \eqref{eqn:Thm3.2_12} can be conservative. We leave the study of tighter theoretical bounds on $K$ for our future research.

The size of the neighborhood $\epsilon$ can be controlled by choosing the number of consensus steps $N_c$. 
Specifically, according to Lemma~\ref{lem:ConsensusErr_1}, the number of steps $N_c$ controls the consensus error $\|\vec{\mu}^k(N_c) - J(\theta_{k} + \delta u_k, \xi_k) \mathbf{1}\|$, which in turn bounds the error $\epsilon$ of the solution; see the proof of Theorem~\ref{thm:DZO} in the Appendix.
Moreover, the number of consensus steps $N_c$ depends not only on the user-specified accuracy level $\epsilon$ and $\epsilon_J$, but also on the range of the policy bounds $J_u$ and $J_l$. This is because the consensus iteration at each episode is independent of those at previous episodes. 
Therefore, to select $N_c$ to control the estimation bias $| \mu_i^k(N_c) - J(\theta_k + \delta u_k, \xi_k) |$, we need to select the term $J_u-J_l$ in the definition of $N_c$ in Theorem~\ref{thm:DZO} as the difference between the initial estimates $\mu_i^k(0) = J_i(\theta_k + \delta u_k, \xi_k) \in [J_l, J_u]$ that have the maximum and minimum value.

\subsection{Distributed Residual-Feedback Zeroth-Order Policy Optimization with Value Tracking}
\label{sec:DZO_ValueTracking}

As discussed in Section~\ref{sec:DZO}, the estimation bias $|\mu_i^k(N_c)-J(\theta_k+\delta u_k, \xi_k)|$ at episode $k$ can be reduced by using local policy estimates from previous episodes.
Specifically, rather than resetting $\mu_i^k(0) = J_i(\theta_k + \delta u_k, \xi_k)$ in line 7 of Algorithm~\ref{alg:DZO}, we update it using the estimate $\mu_i^{k-1}(N_c)$ from the last episode as $\mu_i^k(0) = \mu_i^{k-1}(N_c) + J_i(\theta_k + \delta u_k, \xi_k) - J_i(\theta_{k-1} + \delta u_{k-1}, \xi_{k-1})$.
Then, we run $N_c$ consensus iterations on $\mu_i^k(0)$ as before. Let $\bar{\mu}^k(m) = \frac{1}{N} \sum_{i = 1}^{N} \mu_i^k(m)$. The following lemma shows that the value tracking updates preserve the global information $J(\theta_{k} + \delta u_k, \xi_k)$.
\begin{lemma}
	\label{lem:AvgMu}
	Let Assumption~\ref{asm:Graph} hold. Then, running Algorithm~\ref{alg:DZO} with $\mathtt{DoTracking} = \mathtt{True}$, we have that $\bar{\mu}^k(m) = J(\theta_{k} + \delta u_k, \xi_k) = \frac{1}{N}\sum_{i = 1}^{N} J_i(\theta_{i, k}, \xi_k)$, for all $m = 1, 2, \dots, N_c$ and all $k$.
\end{lemma}
Lemma~\ref{lem:AvgMu} implies that the local estimation bias $|\mu_i^k(N_c) - J(\theta_{k} + \delta u_k, \xi_k) |$ is equal to the consensus error $|\mu_i^k(N_c) - \bar{\mu}^k(N_c)|$. Using value tracking, the bias at episode $k$ can be controlled by the consensus steps of past episodes. This is formally shown in the following lemma. 
\begin{lemma}
	\label{lem:ConsensusErr_2}
	Let Assumptions~\ref{asm:Lipschitz}, \ref{asm:UnbiasedEval}, \ref{asm:Graph} hold and define  $E_\mu^k = \| \vec{\mu}_k(N_c) - \bar{\mu}_k(N_c) \mathbf{1}\|$. Then, running Algorithm~\ref{alg:DZO} with $\mathtt{DoTracking} = \mathtt{True}$, we have that
	 $\mathbb{E}\big[ (E_\mu^k)^2 \big] \leq \bigg( 2\mathbb{E}\big[ (E_\mu^{k-1})^2 \big] + 32dL_0^2\frac{\alpha^2}{\delta^2}\mathbb{E}\big[ (E_\mu^{k-1})^2 \|u_{k-1}\|^2 \big] + 32d^2 L_0^2$ $\frac{\alpha^2}{\delta^2} \mathbb{E}\big[ (E_\mu^{k-2})^2 \big] \bigg) \rho_W^{2N_c} + 16NL_0^2 \alpha^2 \mathbb{E}\big[ \| \tilde{\nabla} J(\theta_{k-1})\|^2 \big] \rho_W^{2N_c} + 32Nd L_0^2 \delta^2 \rho_W^{2N_c} + 16N\sigma^2 \rho_W^{2N_c}$.
\end{lemma}
Compared to Lemma~\ref{lem:ConsensusErr_1}, the proposed value tracking technique makes it possible to bound the consensus error at episode $k$ with the consensus errors from episodes $k-1$ and $k-2$.  Furthermore, at each episode, this error is perturbed by the second order momentum of the policy gradient estimate~\eqref{eqn:Residual} which can be controlled by choosing a small stepsize $\alpha$ and a large number $N_c$. To see the benefit of this result, when the consensus errors from episodes $k-1$ and $k-2$ are small, value tracking needs fewer consensus iterations to achieve small consensus error at episode $k$. This is in contrast to the case without value tracking, where $N_c$ is selected regardless of previous consensus errors. The following result shows convergence of Algorithm~\ref{alg:DZO} using value tracking.
\begin{theorem}
	\label{thm:DZO_ValueTrack}
	{\bf (Learning Rate of Algorithm~\ref{alg:DZO} with Value Tracking)}
Let Assumptions~\ref{asm:Lipschitz}, \ref{asm:UnbiasedEval}, \ref{asm:Graph} hold and define	$\delta = \frac{\epsilon_J}{\sqrt{d}L_0}$, $ \alpha = \frac{\epsilon_J^{1.5}} {4d^{1.5}L_0^2\sqrt{K}}$, and 
$  N_c \geq \max\big( \log(\rho_W)^{-1} \log(\frac{1}{2\sqrt{2}}),  \log(\rho_W)^{-1}$ $ \log(\sqrt{ \frac{\epsilon}{ 4\big(G^2 \epsilon_J + 32 (d+4)^2d L_0^2  + 16 d^3L_0^2\sigma^2/\epsilon_J^2\big)}}) \big)$
where $G^2 = \max \bigg( \mathbb{E}\big[ \| \tilde{\nabla} J(\theta_0) \|^2\big], \frac{2 \epsilon_J \epsilon }{d K} + 32L_0^2(d+4)^2$ $+ 16d^2L_0^2 \frac{\sigma^2}{\epsilon_J^2} \bigg)$. Then, running Algorithm~\ref{alg:DZO} with $\mathtt{DoTracking} = \mathtt{True}$, we have that $\frac{1}{K} \sum_{k = 0}^{K-1} \mathbb{E}[ \| \nabla J_\delta(\theta_k) \|^2 ] \leq \mathcal{O}( d^{1.5} \epsilon_J^{-1.5} K^{-0.5} ) + \frac{\epsilon}{2}$,  where the expectation is taken over the trajectory of the sampled vector $u_k$ and evaluation noise $\xi_k$.
\end{theorem}

In Theorem~\ref{thm:DZO_ValueTrack}, the constant $G^2$ represents the uniform bound on $\mathbb{E}[ \|\tilde{\nabla} J(\theta_{k})\|^2 ]$ for all $k = 1, 2, \dots, K$. 
Moreover, from the bounds on $N_c$ in Theorems~\ref{thm:DZO} and \ref{thm:DZO_ValueTrack}, when $\epsilon$ and $\epsilon_J$ are close to $0$, we obtain that $N_c \sim \mathcal{O}\big(\log( \frac{\sqrt{\epsilon} \epsilon_J}{d^{1.5} \sigma} ) / \log(\rho_W) \big)$ in Theorem~\ref{thm:DZO_ValueTrack} and $N_c \sim \mathcal{O}\big(\log( \frac{\sqrt{\epsilon} \epsilon_J }{d^{1.5} (J_u - J_l)} ) / \log(\rho_W) \big)$ in Theorem~\ref{alg:DZO}.
This suggests that the choice of $N_c$ in Theorem~\eqref{thm:DZO_ValueTrack} depends on the variance of function evaluation $\sigma^2$, while in Theorem~\ref{thm:DZO} the choice of $N_c$ depends on the range of value functions $[J_l, J_u]$. 
In practice, the standard deviation of function evaluation $\sigma$ can be much smaller than the range of its value $[J_l, J_u]$. To see this, note that by Assumption~\ref{asm:BoundVal}, we have that the noisy sample of the objective function value $J_i(\theta, \xi)$ has bounded support. Then, applying the Popoviciu’s inequality on variances, we get that the variance $\sigma^2$ satisfies that $\sigma^2 \leq \frac{1}{4} (J_u - J_l)^2$, that is, $\sigma \leq \frac{1}{2}|J_u - J_l|$.
Therefore, Algorithm~\ref{alg:DZO} with value tracking requires fewer consensus steps per episode than without value tracking.

\begin{rem}
	\label{rmk:onepoint_vt}
	The proposed value-tracking technique can also be combined with the existing distributed one-point policy gradient estmiator \cite{li2019distributed} to reduce the variance of its gradient estimates. To see this, note that the global value function $J(\theta_k + \delta u_k ,\xi_k)$ used in the one-point estimator~\eqref{eqn:OnePoint} can be replaced by the local esimate of the value $J(\theta_k + \delta u_k ,\xi_k)$, i.e., $\mu_i^k(N_c)$. Then, we obtain the following distributed one-point policy gradient estimator with value tracking:
	$ \tilde{\nabla}_{\theta_{i, k}}J(\theta_k) \approx \frac{\mu_i^k(N_c)}{\delta} u_{i,k}  =   u_{i,k} \big( \sum_{j \in \mathcal{N}_i} [W^{N_c}]_{ij} \big( \mu_j^{k-1}(N_c) + J_j(\theta_k + \delta u_k, \xi_k) - J_j(\theta_{k-1}+\delta u_{k-1}, \xi_{k-1})  \big) \big) / \delta. $
	We observe that the estimator  $\tilde{\nabla}_{\theta_{i, k}}J(\theta_k)$ has the same structure as the distributed residual-feedback policy gradient estimator without value tracking~\eqref{eqn:Alg_Residual} except for an additional noise term $\frac{\sum_{j \in \mathcal{N}_i} [W^{N_c}]_{ij} \mu_j^{k-1}(N_c)}{\delta} u_{i,k}$. Therefore, the variance of the estimator $\tilde{\nabla}_{\theta_{i, k}}J(\theta_k)$ is reduced through a similar mechanism as that of the distributed residual-feedback policy gradient estimator without value tracking. As a result, the learning performance is improved compared to that of the existing distributed one-point policy gradient estimator \cite{li2019distributed}, as we will demonstrate in the next section.
\end{rem}

\section{Experiments}
\label{sec:exp}
\begin{figure}[t]
	\centering
	\subfigure[\label{fig:R_OP}]{
		\includegraphics[width=.48\linewidth]{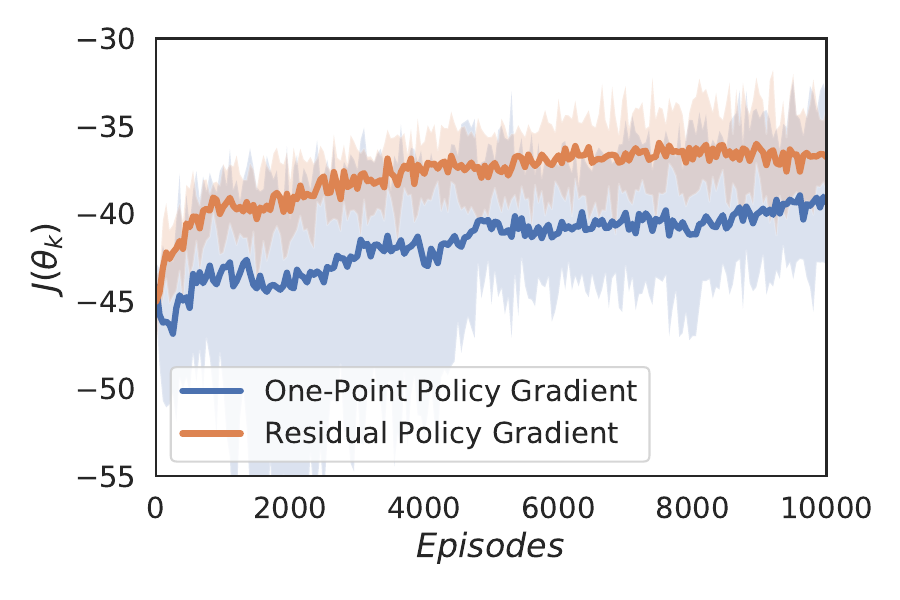}}
	\subfigure[\label{fig:R_OP_Tracking}]{
		\includegraphics[width=.48\linewidth]{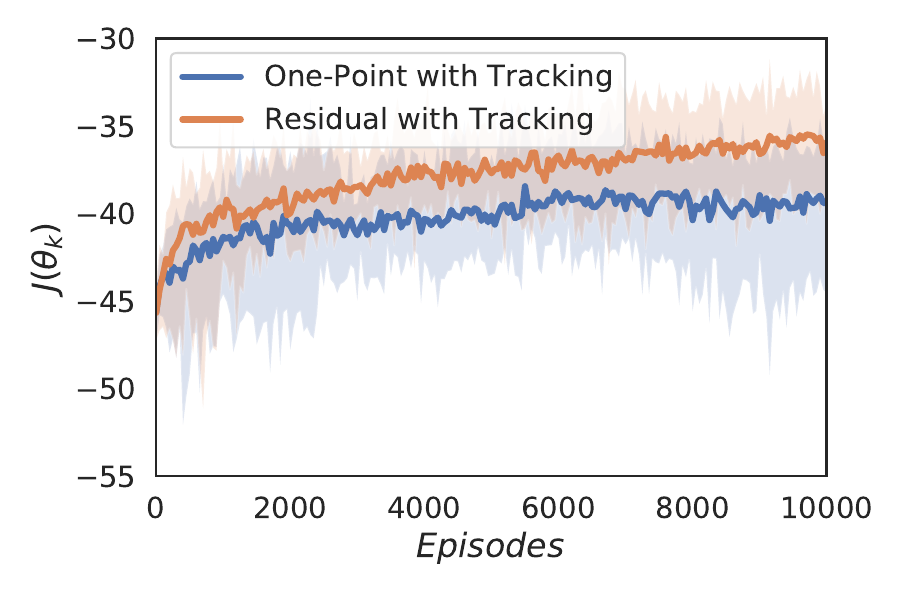}}
	\caption{Distributed zeroth-order policy optimization with the proposed residual-feedback estimator (5) (orange) versus the one-point estimator (3) (blue). In each case, Algorithm 1 is run 10 times. (a): Results without value tracking. (b): Results with value tracking. }
\end{figure}
\begin{figure}[t]
	\centering
	\subfigure[\label{fig:OP_Track}]{
		\includegraphics[width=.48\linewidth]{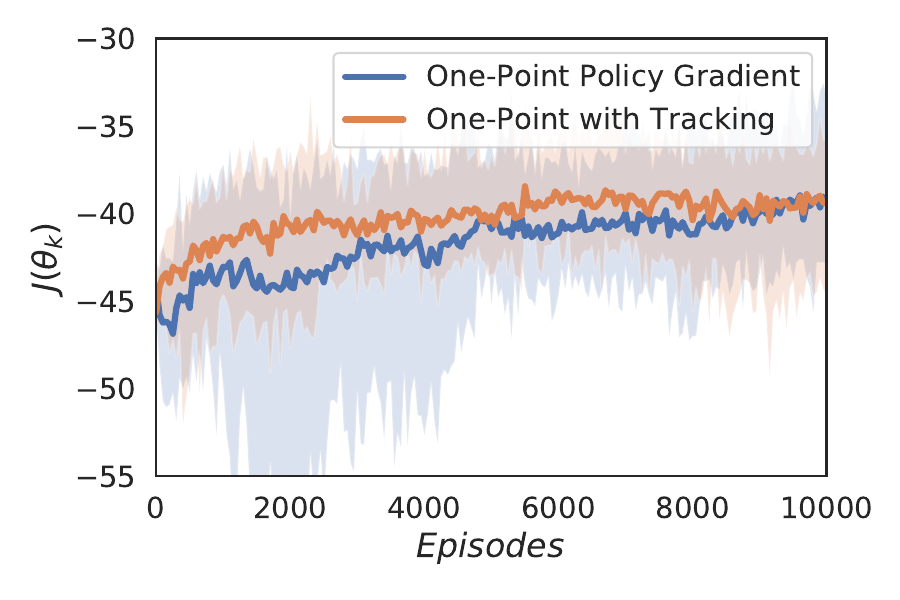}}
	\subfigure[\label{fig:Residual_Track}]{
		\includegraphics[width=.48\linewidth]{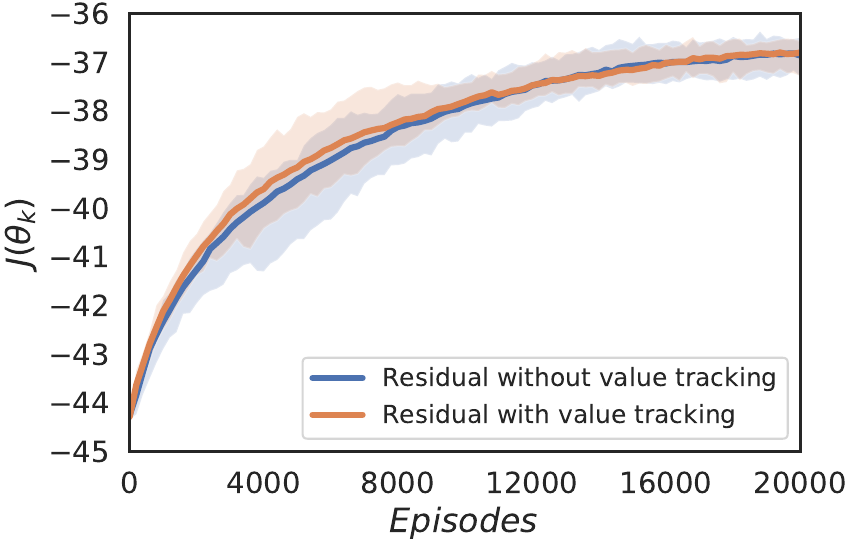}}
	\subfigure[\label{fig:Residual_Track_consensus}]{
	\includegraphics[width=.48\linewidth]{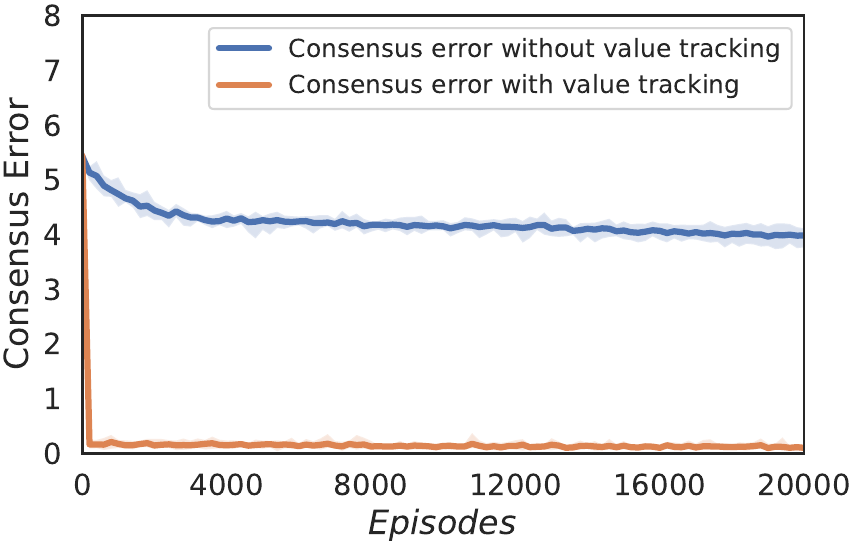}}
	\caption{Distributed zeroth-order policy optimization with value tracking (orange) versus without value tracking (blue). In each case, Algorithm~\ref{alg:DZO} is run 10 times. (a): Comparative results for the one-point estimator (3). (b): Comparative results for the proposed residual-feedback estimator (5). (c) Maximum absolute consensus errors $\max_i |\mu_i^k(N_c) - J(\theta_k + \delta u_k, \xi_k)|$ over episodes. }
\end{figure}
\begin{figure}[t]
	\centering
	\subfigure[\label{fig:RF_Track_Badgraph}]{
	\includegraphics[width=.48\linewidth]{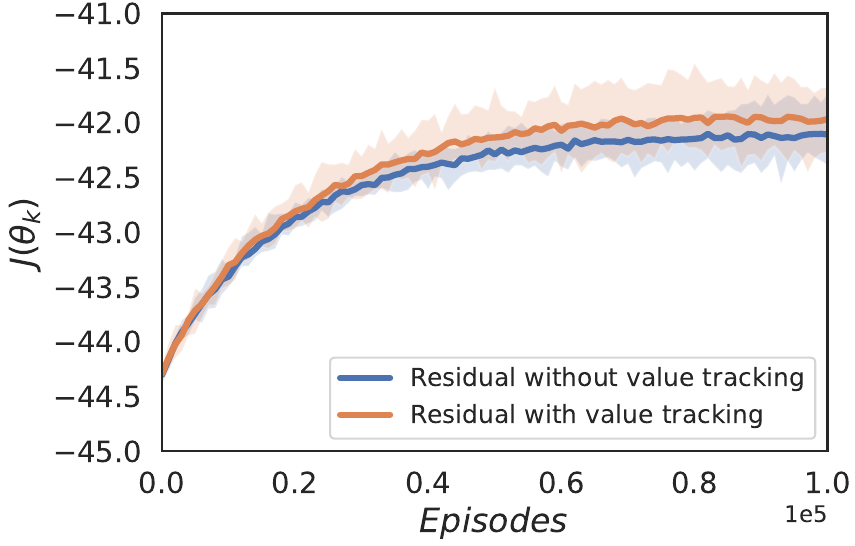}}
	\subfigure[\label{fig:Comparison_Nc}]{
	\includegraphics[width=.48\linewidth]{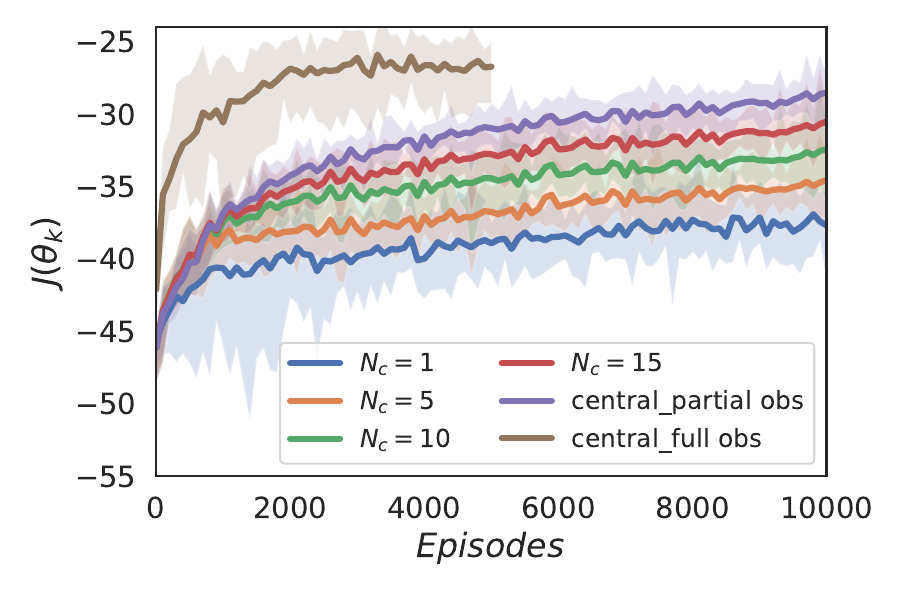}}
	\caption{(a) Algorithm~\ref{alg:DZO} with value tracking (orange) versus without value tracking (blue) under the communication graph which does not respect the coupling relationship among agents. (b) Comparative results for Algorithm~\ref{alg:DZO} with value tracking under different number of consensus steps $N_c$ and the centralized algorithms with partial and full observations. In each case, Algorithm~\ref{alg:DZO} is run 10 times. }
	\label{fig:Residual_Track_Diminishing}
\end{figure}

In this section, we illustrate our proposed MARL algorithm on stochastic multi-agent multi-stage decision making problems. Specifically, we conduct an ablation study to demonstrate the benefits of applying the decentralized residual-feedback zeroth-order policy gradient estimate~\eqref{eqn:Alg_Residual} and the value tracking technique separately. We consider $16$ agents that are located on a $4 \times 4$ grid. 
Agent $i$ has state $s_i(t) = [m_i(t), d_i(t)]$, which denotes the resources it stores and the  local demand it receives in the amount of $m_i(t)$ and $d_i(t)$ at time $t$, respectively. In the meantime, agent $i$ takes actions $a_{ij}(t)$, which denote the resources it shares with its neighbors $j\in\mathcal{N}_i$ in the grid. Specifically, $a_{ij}(t) \in [0, 1]$ denotes the fraction of resources agent $i$ sends to its neighbor $j$ at time $t$.
The local resources and demands at agent $i$ are defined as $m_i(t+1) = m_i(t) - \sum_{j \in \mathcal{N}_i} a_{ij}(t) m_i(t) + \sum_{j \in \mathcal{N}_i} a_{ji}(t) m_j(t) - d_i(t)$ and $d_i(t) = A_i \sin(\omega_i t + \phi_i) + w_{i,t}$,
where $w_{i,t}$ is the noise in the demand. At time $t$, agent $i$ receives a local reward $r_i(t)$, such that $r_i(t) = 0$ when $m_i(t) \geq 0$ and $r_i(t) = -m_i(t)^2$ when $m_i(t) < 0$. 
We consider a partial observation scenario, where agent $i$ can only observe its local resources and demands, that is, $o_i(t) = [m_i(t), d_i(t)]^T$. 
Agent $i$ determines its actions $\{ a_{ij}(t) \}$ using its local policy function $\pi_{i}(o_i(t) | \theta_i)$, where $\theta_i$ is the policy function parameter for agent $i$.
Specifically, we have that $a_{ij} = \exp(z_{ij}) / \sum_{j}\exp(z_{ij})$, where $z_{ij} = \sum_{p = 1}^{9} \|o_i - c_p\|^2 \theta_{ij}(p)$ and $c_p$ is the $p$-th feature parameter. 
We consider episodes of length $T=30$, and select the discount factor as $\gamma = 0.75$. 
The goal of the agents is to find the optimal joint policy parameter $\theta$, so that their total accumulated reward $J(\theta) = \mathbb{E}_{s_0, a_0 \sim \rho_0} \big[ \sum_{t = 0}^{30} \gamma^{t} \sum_{i = 1}^{N} r_i(t) \big | \pi(\theta)]$ is maximized.
The communication graph is assumed to be a chain graph. Moreover, we select the number of consensus steps $N_c = 1$, and show that Algorithm~\ref{alg:DZO} with value tracking can achieve policy improvement even in this challenging scenario. 
The stepsizes are selected so that the convergence speed is optimized. 

First, we compare the performance of Algorithm~\ref{alg:DZO} using the decentralized policy gradients~\eqref{eqn:Residual} and \eqref{eqn:OnePoint}, without value tracking. The learning progress is presented in Figure~\ref{fig:R_OP}. We observe that the decentralized residual-feedback policy gradient estimator has less variance than the existing one-point policy gradient estimator and, therefore, improves faster and finds a better policy in the end of the learning. The same effect is observed in Figure~\ref{fig:R_OP_Tracking}, when both estimators are implemented with value tracking. This suggests that the residual-feedback zeroth-order policy gradient estimator is superior to the one-point policy gradient estimator for decentralized policy optimization problems.

Next, we demonstrate the merit of using value tracking. Specifically, we first run Algorithm~\ref{alg:DZO} with the decentralized one-point policy gradient estimator~\eqref{eqn:OnePoint}, with and without value tracking. The difference in the performance is shown in Figure~\ref{fig:OP_Track}. We observe that using value tracking results in less variance and also achieves better policies. 
This is because the decentralized one-point policy gradient estimator with value tracking $\tilde{\nabla}_{\theta_{i, k}}J(\theta_k)$ has the same variance reduction effect on the policy gradient estimates as the residual feedback estimator~\ref{eqn:Residual}, as we have discussed in Remark~\ref{rmk:onepoint_vt}.
Figure~\ref{fig:Residual_Track} shows the results of using the residual-feedback policy gradient estimator~\eqref{eqn:Residual} with and without value tracking. 
We observe that Algorithm~\ref{alg:DZO} with value tracking performs slightly better in the mean than without value tracking. This is because value tracking can track the value of the global objective function better, as shown in Figure~\ref{fig:Residual_Track_consensus} where the maximum consensus error $\max_i |\mu_i^k(N_c) - J(\theta_k + \delta u_k, \xi_k)|$ at each episode $k$ is presented. 
The improvement achieved by value tracking is not very significant in Figure~\ref{fig:Residual_Track} because the underlying communication graph respects the coupling relationship among agents. 
Specifically, we say that the communication graph respects the coupling relationship between agents if the action of every agent $i$ that can directly communicate with an agent $j$ also directly affects the reward and transition function of that agent $j$.
In this case, the rewards received from an agent's local neighbors can approximate this agent's contribution to the global reward well even without tracking the information from other distant neighbors in the graph.
	
To further demonstrate the advantage of combining the decentralized residual-feedback gradient estimator with value tracking, we consider a challenging scenario where the communication graph does not respect the coupling relationship among agents as described above. The performance of Algorithm~\ref{alg:DZO} using the residual-feedback estimator~\eqref{eqn:Residual} with and without value tracking is presented in Figure~\ref{fig:RF_Track_Badgraph}. 
In this case, using rewards from the local neighbors does not approximate well the local agent's contribution to the global reward. Therefore, value tracking can help obtain a better estimate of the global reward information. As a result, the decentralized residual-feedback policy gradient estimator with value tracking outperforms the one without value tracking, as shown in Figure~\ref{fig:RF_Track_Badgraph}.
 The numerical results presented above show that the performance of Algorithm~\ref{alg:DZO} is affected by the structure of the communication graph among the agents. As shown in Lemma~\ref{lem:ConsensusErr_1}, the second largest singular value of the matrix $W$, $\rho_W$, captures the speed at which the agents can reach consensus on the global objective function value. As shown in (23) in Lemma~\ref{lem:SecondMoment}, the bound on the second moment of the gradient estimate is related to the consensus error terms, e.g., $\|\vec{\mu}^{k-1} - \bar{\mu}^{k-1} \mathbf{1} \|$,. Therefore, the smaller the value of $\rho_W$, the smaller is the second moment and, therefore, the variance of the gradient estimate. This suggests that the variance in learning can be further reduced for communication graphs that enable fast consensus.

Finally, we demonstrate the effect of the consensus steps $N_c$ on Algorithm~\ref{alg:DZO} by comparing to a centralized algorithm that uses the gradient estimator~\eqref{eqn:Residual} with both full and partial observations. Specifically, in the centralized algorithm, the value of the global objective function $J(\theta_k + \delta u_k, \xi_k)$ is directly provided to each local agent at each episode, and the local agents' policy functions receive all agents' states as inputs when full observations are assumed and only receive the neighboring agents' states as inputs when partial observations are assumed. As shown in Figure~\ref{fig:Comparison_Nc}, as the number of consensus steps $N_c$ increases, the performance of Algorithm~\ref{alg:DZO} with value tracking approaches that of the centralized algorithm with partial observations. And the performance of the centralized algorithm with partial observations slightly underperforms that of the centralized algorithm with full observations. This is because policy functions learned using partial observations constitute a subset of those that can be learned using full observations. Note that the centralized algorithm requires all the agents to observe the values of all local objective functions, which is not a scalable approach in practice.
Its performance is provided here as a benchmark to compare to the performance of the proposed distributed algorithm.

\section{Conclusion}
In this paper, we proposed a new distributed zeroth-order policy optimization method for MARL problems. Compared to existing MARL algorithms that require all the agents' states and actions to be accessible by every local agent, our algorithm can be applied even when each agent only observes partial states and actions. Specifically, we developed a new distributed residual-feedback zeroth-order estimator of the policy gradient and analyzed the effect of bias in the local policy gradient estimates on the convergence of the proposed MARL algorithm. 
Furthermore, we introduced a value tracking technique to reduce the number of consensus steps needed at each episode to control the bias in the estimation of the policy gradient. Finally, we provided numerical experiments on a stochastic multi-agent multi-stage decision making problem that demonstrated the effectiveness of both the decentralized residual-feedback policy gradient estimator and the value tracking technique.

\bibliographystyle{IEEEtran}
\bibliography{biblio}

\begin{thebibliography}{10}
\providecommand{\url}[1]{#1}
\csname url@samestyle\endcsname
\providecommand{\newblock}{\relax}
\providecommand{\bibinfo}[2]{#2}
\providecommand{\BIBentrySTDinterwordspacing}{\spaceskip=0pt\relax}
\providecommand{\BIBentryALTinterwordstretchfactor}{4}
\providecommand{\BIBentryALTinterwordspacing}{\spaceskip=\fontdimen2\font plus
\BIBentryALTinterwordstretchfactor\fontdimen3\font minus
  \fontdimen4\font\relax}
\providecommand{\BIBforeignlanguage}[2]{{%
\expandafter\ifx\csname l@#1\endcsname\relax
\typeout{** WARNING: IEEEtran.bst: No hyphenation pattern has been}%
\typeout{** loaded for the language `#1'. Using the pattern for}%
\typeout{** the default language instead.}%
\else
\language=\csname l@#1\endcsname
\fi
#2}}
\providecommand{\BIBdecl}{\relax}
\BIBdecl

\bibitem{gupta2017cooperative}
J.~K. Gupta, M.~Egorov, and M.~Kochenderfer, ``Cooperative multi-agent control
  using deep reinforcement learning,'' in \emph{International Conference on
  Autonomous Agents and Multiagent Systems}.\hskip 1em plus 0.5em minus
  0.4em\relax Springer, 2017, pp. 66--83.

\bibitem{lowe2017multi}
R.~Lowe, Y.~Wu, A.~Tamar, J.~Harb, P.~Abbeel, and I.~Mordatch, ``Multi-agent
  actor-critic for mixed cooperative-competitive environments,'' in
  \emph{Advances in Neural Information Processing Systems}, 2017, pp.
  6379--6390.

\bibitem{foerster2017counterfactual}
J.~Foerster, G.~Farquhar, T.~Afouras, N.~Nardelli, and S.~Whiteson,
  ``Counterfactual multi-agent policy gradients,'' in \emph{Thirty-Second AAAI
  Conference on Artificial Intelligence}, 2018.

\bibitem{omidshafiei2017deep}
S.~Omidshafiei, J.~Pazis, C.~Amato, J.~P. How, and J.~Vian, ``Deep
  decentralized multi-task multi-agent reinforcement learning under partial
  observability,'' in \emph{Proceedings of the 34th International Conference on
  Machine Learning-Volume 70}.\hskip 1em plus 0.5em minus 0.4em\relax JMLR.
  org, 2017, pp. 2681--2690.

\bibitem{zhang2018fullyICML}
K.~Zhang, Z.~Yang, H.~Liu, T.~Zhang, and T.~Basar, ``Fully decentralized
  multi-agent reinforcement learning with networked agents,'' in
  \emph{International Conference on Machine Learning}, 2018, pp. 5867--5876.

\bibitem{ZhangCDC2019_DistributedRL}
Y.~Zhang and M.~M. Zavlanos, ``Distributed off-policy actor-critic
  reinforcement learning with policy consensus,'' in \emph{58th IEEE Conference
  on Decision and Control}, Nice, France, December 2019.

\bibitem{suttle2019multi}
W.~Suttle, Z.~Yang, K.~Zhang, Z.~Wang, T.~Basar, and J.~Liu, ``A multi-agent
  off-policy actor-critic algorithm for distributed reinforcement learning,''
  \emph{arXiv preprint arXiv:1903.06372}, 2019.

\bibitem{heredia2019distributed}
P.~C. Heredia and S.~Mou, ``Distributed multi-agent reinforcement learning by
  actor-critic method,'' \emph{IFAC-PapersOnLine}, vol.~52, no.~20, pp.
  363--368, 2019.

\bibitem{zhang2020improving}
\BIBentryALTinterwordspacing
Y.~Zhang, Y.~Zhou, K.~Ji, and M.~M. Zavlanos, ``A new one-point
  residual-feedback oracle for black-box learning and control,''
  \emph{Automatica}, p. 110006, 2021. [Online]. Available:
  \url{https://www.sciencedirect.com/science/article/pii/S000510982100532X}
\BIBentrySTDinterwordspacing

\bibitem{fazel2018global}
M.~Fazel, R.~Ge, S.~Kakade, and M.~Mesbahi, ``Global convergence of policy
  gradient methods for the linear quadratic regulator,'' in \emph{Proceedings
  of the 35th International Conference on Machine Learning}, vol.~80, 2018.

\bibitem{malik2018derivative}
D.~Malik, A.~Pananjady, K.~Bhatia, K.~Khamaru, P.~L. Bartlett, and M.~J.
  Wainwright, ``Derivative-free methods for policy optimization: Guarantees for
  linear quadratic systems,'' \emph{arXiv preprint arXiv:1812.08305}, 2018.

\bibitem{li2019distributed}
Y.~Li, Y.~Tang, R.~Zhang, and N.~Li, ``Distributed reinforcement learning for
  decentralized linear quadratic control: A derivative-free policy optimization
  approach,'' \emph{arXiv preprint arXiv:1912.09135}, 2019.

\bibitem{flaxman2005online}
A.~D. Flaxman, A.~T. Kalai, and H.~B. McMahan, ``Online convex optimization in
  the bandit setting: gradient descent without a gradient,'' in
  \emph{Proceedings of the sixteenth annual ACM-SIAM symposium on Discrete
  algorithms}.\hskip 1em plus 0.5em minus 0.4em\relax Society for Industrial
  and Applied Mathematics, 2005, pp. 385--394.

\bibitem{nesterov2013introductory}
Y.~Nesterov, \emph{Introductory lectures on convex optimization: A basic
  course}.\hskip 1em plus 0.5em minus 0.4em\relax Springer Science \& Business
  Media, 2013, vol.~87.

\bibitem{duchi2015optimal}
J.~C. Duchi, M.~I. Jordan, M.~J. Wainwright, and A.~Wibisono, ``Optimal rates
  for zero-order convex optimization: The power of two function evaluations,''
  \emph{IEEE Transactions on Information Theory}, vol.~61, no.~5, pp.
  2788--2806, 2015.

\bibitem{hajinezhad2018gradient}
D.~Hajinezhad and M.~M. Zavlanos, ``Gradient-free multi-agent nonconvex
  nonsmooth optimization,'' in \emph{2018 IEEE Conference on Decision and
  Control (CDC)}.\hskip 1em plus 0.5em minus 0.4em\relax IEEE, 2018, pp.
  4939--4944.

\bibitem{sahu2018distributed}
A.~K. Sahu, D.~Jakovetic, D.~Bajovic, and S.~Kar, ``Distributed zeroth order
  optimization over random networks: A kiefer-wolfowitz stochastic
  approximation approach,'' in \emph{2018 IEEE Conference on Decision and
  Control (CDC)}.\hskip 1em plus 0.5em minus 0.4em\relax IEEE, 2018, pp.
  4951--4958.

\bibitem{dixit2020online}
R.~Dixit, A.~S. Bedi, and K.~Rajawat, ``Online learning over dynamic graphs via
  distributed proximal gradient algorithm,'' \emph{IEEE Transactions on
  Automatic Control}, vol.~66, no.~11, pp. 5065--5079, 2020.

\bibitem{bastianello2021distributed}
N.~Bastianello and E.~Dall’Anese, ``Distributed and inexact proximal gradient
  method for online convex optimization,'' in \emph{2021 European Control
  Conference (ECC)}.\hskip 1em plus 0.5em minus 0.4em\relax IEEE, 2021, pp.
  2432--2437.

\bibitem{yuan2021distributed}
D.~Yuan, L.~Wang, A.~Proutiere, and G.~Shi, ``Distributed zeroth-order
  optimization: Convergence rates that match centralized counterpart,'' 2021.

\bibitem{yang2020q}
Y.~Yang, J.~Hao, G.~Chen, H.~Tang, Y.~Chen, Y.~Hu, C.~Fan, and Z.~Wei,
  ``Q-value path decomposition for deep multiagent reinforcement learning,'' in
  \emph{International Conference on Machine Learning}.\hskip 1em plus 0.5em
  minus 0.4em\relax PMLR, 2020, pp. 10\,706--10\,715.

\bibitem{lanctot2017unified}
M.~Lanctot, V.~Zambaldi, A.~Gruslys, A.~Lazaridou, K.~Tuyls, J.~P{\'e}rolat,
  D.~Silver, and T.~Graepel, ``A unified game-theoretic approach to multiagent
  reinforcement learning,'' in \emph{Advances in Neural Information Processing
  Systems}, 2017, pp. 4190--4203.

\bibitem{srinivasan2018actor}
S.~Srinivasan, M.~Lanctot, V.~Zambaldi, J.~P{\'e}rolat, K.~Tuyls, R.~Munos, and
  M.~Bowling, ``Actor-critic policy optimization in partially observable
  multiagent environments,'' in \emph{Advances in neural information processing
  systems}, 2018, pp. 3422--3435.

\bibitem{nesterov2017random}
Y.~Nesterov and V.~Spokoiny, ``Random gradient-free minimization of convex
  functions,'' \emph{Foundations of Computational Mathematics}, vol.~17, no.~2,
  pp. 527--566, 2017.

\bibitem{gasnikov2017stochastic}
A.~V. Gasnikov, E.~A. Krymova, A.~A. Lagunovskaya, I.~N. Usmanova, and F.~A.
  Fedorenko, ``Stochastic online optimization. single-point and multi-point
  non-linear multi-armed bandits. convex and strongly-convex case,''
  \emph{Automation and remote control}, vol.~78, no.~2, pp. 224--234, 2017.

\bibitem{chen2022improve}
X.~Chen, Y.~Tang, and N.~Li, ``Improve single-point zeroth-order optimization
  using high-pass and low-pass filters,'' in \emph{International Conference on
  Machine Learning}.\hskip 1em plus 0.5em minus 0.4em\relax PMLR, 2022, pp.
  3603--3620.

\bibitem{konda2000actor}
V.~R. Konda and J.~N. Tsitsiklis, ``Actor-critic algorithms,'' in
  \emph{Advances in neural information processing systems}, 2000, pp.
  1008--1014.

\bibitem{williams1992simple}
R.~J. Williams, ``Simple statistical gradient-following algorithms for
  connectionist reinforcement learning,'' \emph{Machine learning}, vol.~8, no.
  3-4, pp. 229--256, 1992.

\bibitem{nedic2017achieving}
A.~Nedic, A.~Olshevsky, and W.~Shi, ``Achieving geometric convergence for
  distributed optimization over time-varying graphs,'' \emph{SIAM Journal on
  Optimization}, vol.~27, no.~4, pp. 2597--2633, 2017.

\end{thebibliography}

\appendix

\section{Proof of Lemma~\ref{lem:ConsensusErr_1}}
{\bf Lemma 3.1.} Given Assumptions~\ref{asm:Graph} and \ref{asm:BoundVal}, we have that $\|\vec{\mu}^k(N_c) - J(\theta_{k} + \delta u_k, \xi_k) \mathbf{1}\| \leq \rho_W^{N_c} \sqrt{N} (J_u - J_l)$, where $\rho_W = \|W - \frac{1}{N}\mathbf{1}\mathbf{1}^T\| < 1$. 

\begin{proof}
	First, we show that $\frac{1}{N} \mathbf{1}^T \vec{\mu}^k(m) = J(\theta_k + \delta u_k, \xi_k)$ for all $m = 0, 1, \dots, N_c$. Note that the consensus step $\mu_i^k(m+1) = \sum_{j \in \mathcal{N}_i} W_{ij} \mu_j^k(m)$ (line 12 in Algorithm~\ref{alg:DZO}) can be equivalently written in a compact form as $\vec{\mu}^k(m+1) = W \vec{\mu}^k(m)$. Therefore, we have that
	\begin{align}
		\label{eqn:Lem3.1_1}
		\frac{1}{N} \mathbf{1}^T \vec{\mu}^k(m+1) = \frac{1}{N} \mathbf{1}^T  W \vec{\mu}^k(m) = \frac{1}{N} \mathbf{1}^T  \vec{\mu}^k(m),
	\end{align}
	where the second equality is due to Assumption~\ref{asm:Graph}, that the matrix $W$ is doubly stochastic. Extending equality~\eqref{eqn:Lem3.1_1} from $m$ to $0$, we obtain that $\frac{1}{N} \mathbf{1}^T \vec{\mu}^k(m) = \frac{1}{N} \mathbf{1}^T \vec{\mu}^k(0) = \frac{1}{N} \sum_{i = 1}^N J_i(\theta_k + \delta u_k, \xi_k) = J(\theta_k + \delta u_k, \xi_k)$, for $m = 0, 1, \dots, N_c$.
	
	Next, we show that $\|\vec{\mu}^k(m) - \frac{1}{N} \mathbf{1} \mathbf{1}^T \vec{\mu}^k(m) \| \leq \|W - \frac{1}{N} \mathbf{1} \mathbf{1}^T \|^{m} \|\vec{\mu}^k(0) - \frac{1}{N} \mathbf{1} \mathbf{1}^T \vec{\mu}^k(0) \|$, for $m = 1, 2, \dots, N_c$. To see this, we have that
	\begin{align}
		\label{eqn:Lem3.1_2}
		& \|\vec{\mu}^k(m+1) - \frac{1}{N} \mathbf{1} \mathbf{1}^T \vec{\mu}^k(m+1) \| \nonumber \\
		&  = \|W \vec{\mu}^k(m) - \frac{1}{N} \mathbf{1} \mathbf{1}^T W \vec{\mu}^k(m) \|  \nonumber \\
		& = \|W \vec{\mu}^k(m) - \frac{1}{N} \mathbf{1} \mathbf{1}^T \vec{\mu}^k(m) \|,
	\end{align}
	where the second equality is due to Assumption~\ref{asm:Graph}. According to \eqref{eqn:Lem3.1_2}, we have that $\|\vec{\mu}^k(m+1) - \frac{1}{N} \mathbf{1} \mathbf{1}^T \vec{\mu}^k(m+1) \| = \|(W - \frac{1}{N} \mathbf{1} \mathbf{1}^T) \vec{\mu}^k(m) \| = \|(W - \frac{1}{N} \mathbf{1} \mathbf{1}^T) (\vec{\mu}^k(m) - \frac{1}{N} \mathbf{1} \mathbf{1}^T \vec{\mu}^k(m))\|$. This is because $(W - \frac{1}{N} \mathbf{1} \mathbf{1}^T)  \frac{1}{N} \mathbf{1} \mathbf{1}^T \vec{\mu}^k(m) = 0$. Therefore, we get that
	\begin{align}
		& \|\vec{\mu}^k(m+1) - \frac{1}{N} \mathbf{1} \mathbf{1}^T \vec{\mu}^k(m+1) \| \nonumber \\
		& \leq \|W - \frac{1}{N} \mathbf{1} \mathbf{1}^T \| \| \vec{\mu}^k(m) - \frac{1}{N} \mathbf{1} \mathbf{1}^T \vec{\mu}^k(m) \| \nonumber \\
		& \leq \dots \leq \|W - \frac{1}{N} \mathbf{1} \mathbf{1}^T \|^{m+1} \| \vec{\mu}^k(0) - \frac{1}{N} \mathbf{1} \mathbf{1}^T \vec{\mu}^k(0) \|.
	\end{align}
	The first inequality is due to the definition of the induced matrix norm $\|W - \frac{1}{N} \mathbf{1} \mathbf{1}^T \|$. According to Assumption~\ref{asm:Graph}, we have that $\rho_W = \|W - \frac{1}{N} \mathbf{1} \mathbf{1}^T \| < 1$. Furthermore, recalling the fact that $\frac{1}{N} \mathbf{1}^T \vec{\mu}^k(m) =  \frac{1}{N} \mathbf{1}^T J_i(\theta_k + \delta u_k, \xi_k) = J(\theta_k + \delta u_k, \xi_k)$ for $m = 0, 1, \dots, N_c$, we obtain that
	\begin{align}
		\label{eqn:Lem3.1_3}
		& \|\vec{\mu}^k(N_c) -  J(\theta_k + \delta u_k, \xi_k) \mathbf{1}\| \nonumber \\
		& \leq \rho_W^{N_c} \|\vec{\mu}^k(0) -  J(\theta_k + \delta u_k, \xi_k) \mathbf{1}\|.
	\end{align}
	Since $\mu_i^k(0) = J_i(\theta_k + \delta u_k, \xi_k) \in [J_l, J_u]$ and $J(\theta_k + \delta u_k, \xi_k) = \frac{1}{N} \sum_{i = 1}^{N} J_i(\theta_k + \delta u_k, \xi_k) \in [J_l, J_u]$, we have that $\|\vec{\mu}^k(0) -  J(\theta_k + \delta u_k, \xi_k) \mathbf{1}\| \leq \sqrt{N} (J_u - J_l)$. Plugging this inequality into the bound in \eqref{eqn:Lem3.1_3}, we complete the proof.
\end{proof}

\section{Proof of Theorem~\ref{thm:DZO}}
In the subsequent proof, we need the following lemma by \cite{nesterov2017random}.
\begin{lemma}
	\label{lem:LipschitzGradient}
	(Lipschitz Properties of the Smoothed Function, \cite{nesterov2017random} Given Assumption~\ref{asm:Lipschitz}, the smoothed function $J_\delta(\theta)$ is differentiable and its gradient is Lipschitz, that is, $\|\nabla J_\delta(\theta_1) - \nabla J_\delta(\theta_2) \| \leq \frac{\sqrt{d} L_0}{\delta} \|\theta_1 - \theta_2 \|, \l \text{for all } \theta_1, \theta_2 \in \mathbb{R}^d$. 
\end{lemma}
Next, we present a lemma that bounds the squared norm of the gradient of the smoothed function $\nabla J_\delta (\theta_k)$ at iterate $\theta_k$.

\begin{lemma}
	\label{lem:GradientNorm}
	Let Assumptions~\ref{asm:Lipschitz} and \ref{asm:UnbiasedEval} hold. Then, for all $k \geq 0$, we have that 
	\begin{align}
	\label{eqn:Thm3.2}
	\Exp[ \| \nabla J_\delta (& \theta_k) \|^2 | \mathcal{F}_{k-1} ] \leq  \frac{2}{\alpha} (\Exp[ J_\delta(\theta_{k+1}) - J_\delta(\theta_{k}) | \mathcal{F}_{k-1} ]) \nonumber \\
	& +  2\sqrt{d}L_0 \frac{\alpha}{\delta} \Exp[\|g_\delta(\theta_k)\|^2 | \mathcal{F}_{k-1} ] \nonumber \\
	&  + \frac{d_i}{\delta^2} \Exp[\|\vec{\mu}^k - \bar{\mu}^{k}\mathbf{1} \|^2 \|u_k\|^2 | \mathcal{F}_{k-1} ] \nonumber \\
	& + 4\sqrt{d} d_i L_0 \frac{\alpha}{\delta^3} \Exp[\|\vec{\mu}^k - \bar{\mu}^{k}\mathbf{1} \|^2 \|u_k\|^2 | \mathcal{F}_{k-1} ] \nonumber \\
	& +  4 d^{1.5} d_i L_0 \frac{\alpha}{\delta^3} \Exp[\|\mu^{k-1} - \bar{\mu}^{k-1}\mathbf{1} \|^2 | \mathcal{F}_{k-1} ],
	\end{align}
	where the filtration $\mathcal{F}_{k-1} = \sigma(u_t, \xi_t | t \leq k-1)$, $g_\delta(\theta_k) =\frac{J(\theta_k + \delta u_k, \xi_k) - J(\theta_{k-1} + \delta u_{k-1}, \xi_{k-1})}{\delta} u_k$ and $\bar{\mu}^k = \frac{1}{N} \mathbf{1}^T \vec{\mu}^k = J(\theta_k + \delta u_k, \xi_k)$.
\end{lemma}
\begin{proof}
	According to Assumption~\ref{asm:Lipschitz} and Lemma~\ref{lem:LipschitzGradient}, we have that the smoothed function $f_\delta(\theta)$ has Lipschitz gradient with the Lipschitz constant $L_{1, \delta} = \frac{\sqrt{d}L_0}{\delta}$. Therefore, using the inequality (6) in \cite{nesterov2017random}, we obtain that
	\begin{align}
		\label{eqn:Thm3.2_1}
		\langle \nabla J_\delta (\theta_k) , \theta_{k+1} - \theta_{k} \rangle \leq & \; J_\delta(\theta_{k+1}) - J_\delta(\theta_{k}) + \nonumber \\
		& \frac{L_{1, \delta}}{2} \|\theta_{k+1} - \theta_k\|^2.
	\end{align}
	Without loss of generality, we assume that each agent's local policy function $\pi_i$ is parameterized with $\theta_i \in \mathbb{R}^{d_i}$ and $d_i = \frac{d}{N}$ for all $i$. 
	Then, the update~\eqref{eqn:Alg_Residual} can be written in the compact form
	\begin{align}
		\label{eqn:Thm3.2_2}
		\theta_{k+1} &\; = \theta_{k} + \frac{\alpha}{\delta} \diag\big([\mu_1^k(N_c) - \mu_1^{k-1}(N_c), \nonumber \\
		& \quad \quad \quad \quad \quad \dots, \mu_N^k(N_c) - \mu_N^{k-1}(N_c)]\big) \otimes I_{d_i} u_k \nonumber \\
		& = \theta_k + \frac{\alpha}{\delta} \diag\big( \vec{\mu}^k - \vec{\mu}^{k-1}\big) \otimes I_{d_i} u_k
	\end{align}
	where $\otimes$ represents the kronecker product and $I_{d_i}$ is an identity matrix with dimension $d_i$. To simplify the notation, we use $\mu_i^k$ or $\vec{\mu}^k$ to denote $\mu_i^k(N_c)$ and $\vec{\mu}^k(N_c)$, respectively. Then, we equivalently rewrite equality~\eqref{eqn:Thm3.2_2} as
	\begin{align}
		& \theta_{k+1} = \; \theta_k + \frac{\alpha}{\delta} \diag\big( \vec{\mu}^k  - \bar{\mu}^k \mathbf{1} + J(\theta_k + \delta u_k, \xi_k) \mathbf{1} \nonumber \\ 
		& -\vec{\mu}^{k-1}  + \bar{\mu}^{k-1} \mathbf{1} - J(\theta_{k-1} + \delta u_{k-1}, \xi_{k-1}) \mathbf{1} ) \big) \otimes I_{d_i} u_k, \nonumber
	\end{align}
	because  $\bar{\mu}^k = \frac{1}{N} \mathbf{1}^T \vec{\mu}^k = J(\theta_k + \delta u_k, \xi_k)$ as in the proof of Lemma~\ref{lem:ConsensusErr_1}. Rearranging terms in above equality, we have that
	\begin{align}
	\label{eqn:Thm3.2_3}
		\theta_{k+1} - \theta_{k} = & \; \alpha \frac{J(\theta_k + \delta u_k, \xi_k) - J(\theta_{k-1} + \delta u_{k-1}, \xi_{k-1})}{\delta} u_k  \nonumber \\
		& + \frac{\alpha}{\delta} \diag(\vec{\mu}^k - \bar{\mu}^{k} \mathbf{1}) \otimes I_{d_i} u_k \nonumber \\
		&  - \frac{\alpha}{\delta} \diag(\vec{\mu}^{k-1} - \bar{\mu}^{k-1} \mathbf{1}) \otimes I_{d_i} u_k,
	\end{align}	
	Substituting \eqref{eqn:Thm3.2_3} into the bound in \eqref{eqn:Thm3.2_1} and rearranging terms, we get that
	\begin{align}
		\label{eqn:Thm3.2_4}
		& \alpha \langle \nabla J_\delta (\theta_k) , g_\delta(\theta_k) \rangle \nonumber \\
		& \leq J_\delta(\theta_{k+1}) - J_\delta(\theta_{k}) - \frac{\alpha}{\delta}  \langle \nabla J_\delta (\theta_k), \diag(\vec{\mu}^k - \bar{\mu}^{k} \mathbf{1}) \otimes I_{d_i} u_k \rangle  \nonumber \\
		& + \frac{L_{1, \delta}}{2} \alpha^2 \|g_\delta(\theta_k) + \frac{1}{\delta} \diag(\vec{\mu}^k - \bar{\mu}^{k} \mathbf{1}) \otimes I_{d_i} u_k \nonumber \\
		& - \frac{1}{\delta} \diag(\mu^{k-1} - \bar{\mu}^{k-1} \mathbf{1}) \otimes I_{d_i} u_k\|^2 \nonumber \\
		&  + \frac{\alpha}{\delta}  \langle \nabla J_\delta (\theta_k), \diag(\mu^{k-1} - \bar{\mu}^{k-1} \mathbf{1}) \otimes I_{d_i} u_k \rangle,
	\end{align}
	where $g_\delta(\theta_k) =\frac{J(\theta_k + \delta u_k, \xi_k) - J(\theta_{k-1} + \delta u_{k-1}, \xi_{k-1})}{\delta} u_k$. Dividing both sides of~\eqref{eqn:Thm3.2_4} by $\alpha$ and taking the expectation of both sides with respect to $u_k$ and $\xi_k$ conditioned on the filtration $\mathcal{F}_{k-1} = \sigma( u_t, \xi_t | t \leq k-1)$, we have that 
	\begin{align}
	\label{eqn:Thm3.2_5}
	& \Exp[ \| \nabla J_\delta (\theta_k) \|^2 | \mathcal{F}_{k-1}]  \leq \frac{\Exp[ J_\delta(\theta_{k+1}) - J_\delta(\theta_{k}) | \mathcal{F}_{k-1} ]}{\alpha} \nonumber \\
	& - \frac{1}{\delta}  \Exp[ \langle \nabla J_\delta (\theta_k), \diag(\vec{\mu}^k - \bar{\mu}^{k} \mathbf{1}) \otimes I_{d_i} u_k \rangle | \mathcal{F}_{k-1} ]  & \nonumber \\
	& + \frac{L_{1, \delta}}{2} \alpha \Exp[ \|g_\delta(\theta_k) + \frac{1}{\delta} \diag(\vec{\mu}^k - \bar{\mu}^{k} \mathbf{1}) \otimes I_{d_i} u_k \nonumber \\
	& - \frac{1}{\delta} \diag(\mu^{k-1} - \bar{\mu}^{k-1} \mathbf{1}) \otimes I_{d_i} u_k\|^2 | \mathcal{F}_{k-1} ].
	\end{align}
	Note that all the expectation in the rest of this proof shall be conditional on the filtration $\mathcal{F}_{k-1}$. To simplify notation, in what follows we omit conditioning on the filtration.
	This is because $\Exp[g_\delta(\theta_k)] = \nabla J_\delta (\theta_k)$, $\nabla J_\delta (\theta_k)$ and $\diag(\mu^{k-1} - \bar{\mu}^{k-1} \mathbf{1})$ are fixed when conditioned on the filtration $\mathcal{F}_{k-1}$, and $\Exp[u_k] = 0$. Next, we provide bounds on the second and third terms in the right hand side (RHS) of \eqref{eqn:Thm3.2_5}. Specifically, because $- \langle \nabla J_\delta (\theta_k), \frac{1}{\delta} \diag(\vec{\mu}^k - \bar{\mu}^{k} \mathbf{1}) \otimes I_{d_i} u_k \rangle \leq \frac{1}{2} \| \nabla J_\delta (\theta_k)\|^2 + \frac{1}{2\delta^2} \|\diag(\vec{\mu}^k - \bar{\mu}^{k} \mathbf{1}) \otimes I_{d_i} u_k\|^2$, we have that
	\begin{align}
		\label{eqn:Thm3.2_6}
		& - \frac{1}{\delta}  \Exp[ \langle \nabla J_\delta (\theta_k), \diag(\vec{\mu}^k - \bar{\mu}^{k} \mathbf{1}) \otimes I_{d_i} u_k \rangle] \nonumber \\
		& \leq \frac{1}{2} \Exp[\| \nabla J_\delta (\theta_k)\|^2] + \frac{1}{2\delta^2} \Exp[ \|\diag(\vec{\mu}^k - \bar{\mu}^{k} \mathbf{1}) \otimes I_{d_i} u_k\|^2] \nonumber \\
		& \leq \frac{1}{2} \Exp[\| \nabla J_\delta (\theta_k)\|^2]  + \frac{d_i}{2\delta^2} \Exp[ \|\vec{\mu}^k - \bar{\mu}^k \mathbf{1} \|^2 \|u_k\|^2 ],
	\end{align}
	where the third inequality is due to the fact that $\|\diag(v_1) v_2\|^2 \leq \|v_1\|^2 \|v_2\|^2$ for all $v_1, v_2 \in \mathbb{R}^d$. Furthermore, we have that
	\begin{align}
		\label{eqn:Thm3.2_7}
		&\Exp[ \|g_\delta(\theta_k) + \frac{1}{\delta} \diag(\vec{\mu}^k - \bar{\mu}^{k} \mathbf{1}) \otimes I_{d_i} u_k \nonumber \\
		& \quad - \frac{1}{\delta} \diag(\mu^{k-1} - \bar{\mu}^{k-1} \mathbf{1}) \otimes I_{d_i} u_k\|^2] \nonumber \\
		& \leq 2 \Exp[ \|g_\delta(\theta_k) \|^2] + \frac{2}{\delta^2} \Exp[ \| \diag(\vec{\mu}^k - \bar{\mu}^{k} \mathbf{1}) \otimes I_{d_i} u_k \nonumber \\
		& \quad - \diag(\mu^{k-1} - \bar{\mu}^{k-1} \mathbf{1}) \otimes I_{d_i} u_k \|^2 ] \nonumber \\
		& \leq 2 \Exp[ \|g_\delta(\theta_k) \|^2] + \frac{4}{\delta^2} \Exp[ \| \diag(\vec{\mu}^k - \bar{\mu}^{k} \mathbf{1}) \otimes I_{d_i} u_k\|^2] \nonumber \\
		& \quad  + \frac{4}{\delta^2} \Exp[ \| \diag(\mu^{k-1} - \bar{\mu}^{k-1} \mathbf{1}) \otimes I_{d_i} u_k \|^2] \nonumber \\
		& \leq 2 \Exp[ \|g_\delta(\theta_k) \|^2] + \frac{4d_i}{\delta^2} \Exp[\|\vec{\mu}^k - \bar{\mu}^{k}\mathbf{1} \|^2 \|u_k\|^2] \nonumber \\
		& \quad + \frac{4d d_i}{\delta^2} \Exp[ \| \mu^{k-1} - \bar{\mu}^{k-1} \mathbf{1} \|^2],
	\end{align}
	where the first two inequalities are due to the fact that $\|v_1 + v_2\|^2 \leq 2\|v_1\|^2 + 2\|v_2\|^2$ and the third inequality is due to the fact that $\Exp[ \| \diag(\mu^{k-1} - \bar{\mu}^{k-1} \mathbf{1}) \otimes I_{d_i} u_k \|^2] \leq d_i \Exp[ \| \diag(\mu^{k-1} - \bar{\mu}^{k-1} \mathbf{1}) \|^2 \| u_k \|^2]$, the fact that $\mu^{k-1} - \bar{\mu}^{k-1} \mathbf{1}$ is independent of $u_k$ and the fact that $\Exp[\|u_k\|^2] = d$. Substituting the bounds in \eqref{eqn:Thm3.2_6} and \eqref{eqn:Thm3.2_7} into the bound in \eqref{eqn:Thm3.2_5} and rearranging the terms, we get that
	\begin{align}
		& \Exp[ \| \nabla J_\delta (\theta_k) \|^2] \nonumber \\
		& \leq \frac{2}{\alpha} (\Exp[ J_\delta(\theta_{k+1}) - J_\delta(\theta_{k})]) + 2\sqrt{d}L_0 \frac{\alpha}{\delta} \Exp[\|g_\delta(\theta_k)\|^2] \nonumber \\
		& + 4\sqrt{d} d_i L_0 \frac{\alpha}{\delta^3} \Exp[\|\vec{\mu}^k - \bar{\mu}^{k}\mathbf{1} \|^2 \|u_k\|^2] \nonumber \\
		& + 4 d^{1.5} d_i L_0 \frac{\alpha}{\delta^3} \Exp[\|\mu^{k-1} - \bar{\mu}^{k-1}\mathbf{1} \|^2] \nonumber \\
		& + \frac{d_i}{\delta^2} \Exp[\|\vec{\mu}^k - \bar{\mu}^{k}\mathbf{1} \|^2 \|u_k\|^2].
	\end{align}
	The proof is complete.
\end{proof}
Next, we present a lemma bounding the second moment of the gradient estimate $g_\delta(\theta_k)$.
\begin{lemma}
	\label{lem:SecondMoment}
	Let Assumptions~\ref{asm:Lipschitz} and \ref{asm:UnbiasedEval} hold. Then, for all $k \geq 1$, we have that
	\begin{align}
\label{eqn:LemB.3}
& \Exp[ \|g_\delta(\theta_{k})\|^2 ] \leq 8dL_0^2 \frac{\alpha^2}{\delta^2} \Exp[ \|g_\delta(\theta_{k-1})\|^2 ]\nonumber \\
&  + 16dd_i L_0^2 \frac{\alpha^2}{\delta^4} \Exp[ \|\mu^{k-1} - \bar{\mu}^{k-1} \mathbf{1} \|^2 \|u_{k-1}\|^2 ]  \\
& + 16d^2 d_i L_0^2 \frac{\alpha^2}{\delta^4} \Exp[ \|\mu^{k-2} - \bar{\mu}^{k-2} \mathbf{1} \|^2] + 16(d+4)^2 L_0^2 + \frac{8d\sigma^2}{\delta^2}. \nonumber 
\end{align}	
\end{lemma}
\begin{proof}
	Recalling that $g_\delta(\theta_k) =\frac{J(\theta_k + \delta u_k, \xi_k) - J(\theta_{k-1} + \delta u_{k-1}, \xi_{k-1})}{\delta} u_k$, we have that
	\begin{align}
		\label{eqn:LemB.3_1}
		& \Exp[ \|g_\delta(\theta_{k})\|^2 ] \\
		& = \frac{1}{\delta^2}\Exp[ | J(\theta_k + \delta u_k, \xi_k) - J(\theta_{k-1} + \delta u_{k-1}, \xi_{k-1})|^2 \|u_k\|^2 ]. \nonumber
	\end{align}
	In addition, using the inequality $(a + b)^2 \leq 2a^2 + 2b^2$, we have that
	\begin{align}
		& \quad | J(\theta_k + \delta u_k, \xi_k) - J(\theta_{k-1} + \delta u_{k-1}, \xi_{k-1})|^2 \nonumber \\
		& = |J(\theta_k + \delta u_k, \xi_k)  -  J(\theta_{k-1} + \delta u_{k-1}, \xi_k)  \nonumber \\
		& \quad +  J(\theta_{k-1} + \delta u_{k-1}, \xi_k) - J(\theta_{k-1} + \delta u_{k-1}, \xi_{k-1}) | \nonumber \\
		& \leq 2(J(\theta_k + \delta u_k, \xi_k)  -  J(\theta_{k-1} + \delta u_{k-1}, \xi_k))^2 \nonumber \\
		& \quad + 2(J(\theta_{k-1} + \delta u_{k-1}, \xi_k) - J(\theta_{k-1} + \delta u_{k-1}, \xi_{k-1}))^2. \nonumber
	\end{align}
	Then, by adding and subtracting the term $J(\theta_{k-1} + \delta u_{k}, \xi_{k})$ within the term $(J(\theta_k + \delta u_k, \xi_k)  -  J(\theta_{k-1} + \delta u_{k-1}, \xi_k))^2$ in above inequality and applying the bound $(a + b)^2 \leq 2a^2 + 2b^2$ similarly as above, we get that
	\begin{align}
		\label{eqn:LemB.3_2}
		& | J(\theta_k + \delta u_k, \xi_k) - J(\theta_{k-1} + \delta u_{k-1}, \xi_{k-1})|^2 \nonumber \\
		& \leq 4\big( J(\theta_k + \delta u_k, \xi_k) - J(\theta_{k-1} + \delta u_{k}, \xi_{k}) \big)^2 \\
		& + 4 \big( J(\theta_{k-1} + \delta u_k, \xi_k) - J(\theta_{k-1} + \delta u_{k-1}, \xi_{k}) \big)^2 \nonumber \\
		& + 2  \big( J(\theta_{k-1} + \delta u_{k-1}, \xi_k) - J(\theta_{k-1} + \delta u_{k-1}, \xi_{k-1})\big)^2. \nonumber
	\end{align}
	According to Assumption~\ref{asm:Lipschitz}, we have that $\big( J(\theta_k + \delta u_k, \xi_k) - J(\theta_{k-1} + \delta u_{k}, \xi_{k}) \big)^2 \leq L_0^2 \|\theta_k - \theta_{k-1}\|^2$ and $\big( J(\theta_{k-1} + \delta u_k, \xi_k) - J(\theta_{k-1} + \delta u_{k-1}, \xi_{k}) \big)^2 \leq L_0^2 \delta^2 \|u_k - u_{k-1}\|^2$. Furthermore, according to Assumption~\ref{asm:UnbiasedEval}, we have that $\big( J(\theta_{k-1} + \delta u_{k-1}, \xi_k) - J(\theta_{k-1} + \delta u_{k-1}, \xi_{k-1})\big)^2 \leq 4\sigma^2$. Applying the above bounds to the RHS of \eqref{eqn:LemB.3_2}, we get that
	\begin{align}
	\label{eqn:LemB.3_3}
	& | J(\theta_k + \delta u_k, \xi_k) - J(\theta_{k-1} + \delta u_{k-1}, \xi_{k-1})|^2 \nonumber \\
	&  \leq 4L_0^2 \|\theta_k - \theta_{k-1}\|^2 + 4 L_0^2 \delta^2 \|u_k - u_{k-1}\|^2 + 8\sigma^2.
	\end{align}	
	Substituting the bound \eqref{eqn:LemB.3_3} into \eqref{eqn:LemB.3_1}, we have that
	\begin{align}
		\label{eqn:LemB.3_4}
		\Exp[ \|g_\delta(\theta_{k})\|^2 ] & \leq \frac{4L_0^2}{\delta^2} \Exp[ \|\theta_k - \theta_{k-1}\|^2 \|u_k\|^2 ] \nonumber \\
		&  + 4L_0^2 \Exp[  \|u_k - u_{k-1}\|^2 \|u_k\|^2 ] + \frac{8\sigma^2}{\delta^2} \Exp[\|u_k\|^2] \nonumber \\
		\leq \frac{4 d L_0^2}{\delta^2}  \Exp[ \|\theta_k  & - \theta_{k-1}\|^2 ] + 16 (d+4)^2 L_0^2 + \frac{8d\sigma^2}{\delta^2}.
	\end{align}
	The second inequality is due to the fact that $\|\theta_k - \theta_{k-1}\|^2$ is independent of $u_k$, $\Exp[  \|u_k - u_{k-1}\|^2 \|u_k\|^2 ] \leq 4(d+4)^2$ and $\Exp[\|u_k\|^2] = d$. Specifically, we have that $\Exp[  \|u_k - u_{k-1}\|^2 \|u_k\|^2 ] \leq 4(d+4)^2$ because $\|u_k - u_{k-1}\|^2 \leq 2\|u_k\|^2 + 2 \|u_{k-1}\|^2$  and that $\Exp[ \|u_k\|^4] \leq (d+4)^2$ according to \cite{nesterov2017random}. Substituting the expression for $\theta_k - \theta_{k-1}$ in \eqref{eqn:Thm3.2_3} into \eqref{eqn:LemB.3_4} and applying the bound in \eqref{eqn:Thm3.2_7}, we obtain that
	\begin{align}
		& \Exp[ \|g_\delta(\theta_{k})\|^2 ] \leq 8dL_0^2 \frac{\alpha^2}{\delta^2} \Exp[ \|g_\delta(\theta_{k-1})\|^2 ] \nonumber \\
		& + 16dd_i L_0^2 \frac{\alpha^2}{\delta^4} \Exp[ \|\mu^{k-1} - \bar{\mu}^{k-1} \mathbf{1} \|^2 \|u_{k-1}\|^2 ] \nonumber \\
		& + 16d^2 d_i L_0^2 \frac{\alpha^2}{\delta^4} \Exp[ \|\mu^{k-2} - \bar{\mu}^{k-2} \mathbf{1} \|^2] + 16(d+4)^2 L_0^2 + \frac{8d\sigma^2}{\delta^2}. \nonumber
	\end{align}
	The proof is complete.
\end{proof}
Now, we are ready to present the proof for Theorem~\ref{thm:DZO}.

{\bf Theorem 3.2.} 	{\bf (Learning Rate of Algorithm~\ref{alg:DZO} without Value Tracking)} Let Assumptions~\ref{asm:Lipschitz}, \ref{asm:UnbiasedEval}, \ref{asm:Graph} and \ref{asm:BoundVal} hold and define $\delta = \frac{\epsilon_J}{\sqrt{d}L_0}$, $  \alpha = \frac{\epsilon_J^{1.5}} {4d^{1.5}L_0^2\sqrt{K}}$, and $N_c \geq \log( \frac{\sqrt{\epsilon} \epsilon_J} { \sqrt{2} d^{1.5} L_0 (J_u - J_l))} ) / \log(\rho_W)$. Then, running Algorithm~\ref{alg:DZO} with $\mathtt{DoTracking} = \mathtt{False}$, we have that $\frac{1}{K} \sum_{k = 0}^{K-1} \mathbb{E}[ \| \nabla J_\delta(\theta_k) \|^2 ] \leq \mathcal{O}( d^{1.5} \epsilon_J^{-1.5} K^{-0.5} ) + \frac{\epsilon}{2}$.

\begin{proof}
	According to Lemma~\ref{lem:ConsensusErr_1}, and using Assumptions~\ref{asm:Graph} and \ref{asm:BoundVal}, we select $N_c \geq \log( \frac{\sqrt{\epsilon} \epsilon_J} { \sqrt{2} d^{1.5} L_0 (J_u - J_l))} ) / \log(\rho_W)$ so that $\|\vec{\mu}^k - \bar{\mu}^k \mathbf{1}\| = \| \vec{\mu}^k - J(\theta_k + \delta u_k, \xi_k) \mathbf{1}\| \leq E_\mu$ regardless of $u_k$ for all $k \geq 0$, where $E_\mu$ is a small constant such that $E_\mu^2 = \frac{\epsilon \delta^2}{2 dd_i}$.
	Therefore, the bound in \eqref{eqn:Thm3.2} can be simplified as
	\begin{align}
		& \Exp[ \| \nabla J_\delta (\theta_k) \|^2 | \mathcal{F}_{k-1}] \leq \frac{2}{\alpha} (\Exp[ J_\delta(\theta_{k+1}) - J_\delta(\theta_{k}) | \mathcal{F}_{k-1}]) \\
		& + 2\sqrt{d}L_0 \frac{\alpha}{\delta} \Exp[\|g_\delta(\theta_k)\|^2 | \mathcal{F}_{k-1}]  +  8 d^{1.5} d_i L_0 \frac{\alpha}{\delta^3} E_\mu^2 + \frac{d d_i}{\delta^2} E_\mu^2. \nonumber 
	\end{align}
	Applying the tower rule of the conditional expectation and telescoping the above inequality from $k = 0$ to $K-1$, we get that
	\begin{align}
	\label{eqn:Thm3.2_8}
	& \sum_{k=0}^{K-1}\Exp[ \| \nabla J_\delta (\theta_k) \|^2] \leq \frac{2}{\alpha} (\Exp[ J_\delta(\theta_{K}) - J_\delta(\theta_{0})])  \\
	&  + 2\sqrt{d}L_0 \frac{\alpha}{\delta} \sum_{k=0}^{K-1} \Exp[\|g_\delta(\theta_k)\|^2]  +  8 d^{1.5} d_i L_0 \frac{\alpha}{\delta^3} E_\mu^2 K  + \frac{d d_i}{\delta^2} E_\mu^2 K, \nonumber
	\end{align}	
	where the expectation is taken over the trajectory of random samples of $u_k$ and $\xi_k$.
	Next, we bound the term $\sum_{k=0}^{K-1} \Exp[\|g_\delta(\theta_k)\|^2]$ on the RHS of \eqref{eqn:Thm3.2_8}. Specifically, since $N_c$ is selected so that $\|\vec{\mu}^k - \bar{\mu}^k \mathbf{1}\| = \| \vec{\mu}^k - J(\theta_k + \delta u_k, \xi_k) \mathbf{1}\| \leq E_\mu$ regardless of $u_k$ for all $k \geq 0$, the bound in \eqref{eqn:LemB.3} can be simplified as
	\begin{align}
	\Exp[ \|g_\delta(\theta_{k})\|^2 | \mathcal{F}_{k-1} ] \leq & \; 8dL_0^2 \frac{\alpha^2}{\delta^2} \Exp[ \|g_\delta(\theta_{k-1})\|^2  | \mathcal{F}_{k-1} ] \nonumber \\
	+ 32d^2d_i L_0^2 & \frac{\alpha^2}{\delta^4} E_\mu^2 + 16(d+4)^2 L_0^2 + \frac{8d\sigma^2}{\delta^2}.
	\end{align}
	Applying the tower rule of conditional expectation and telescoping the above inequality from $k=1$ to $K-1$, adding $\Exp[ \|g_\delta(\theta_{0})\|^2 ]$ on both sides, adding $8dL_0^2 \frac{\alpha^2}{\delta^2} \Exp[ \|g_\delta(\theta_{K-1})\|^2 ]$ on the RHS, and rearranging the terms, we have that
	\begin{align}
	\label{eqn:Thm3.2_9}
	\sum_{k=0}^{K-1} \Exp[ \|g_\delta(\theta_{k})& \|^2 ] \leq \frac{1}{1 - \alpha_g} \Exp[ \|g_\delta(\theta_{0})\|^2 ] + \frac{32d^2d_i L_0^2}{1 - \alpha_g} \frac{\alpha^2}{\delta^4} E_\mu^2K  \nonumber \\
	& + \frac{16(d+4)^2 L_0^2}{1 - \alpha_g} K + \frac{8d\sigma^2}{(1-\alpha_g)\delta^2} K,
	\end{align}	
	where $\alpha_g = 8dL_0^2 \frac{\alpha^2}{\delta^2}$. When $\delta = \frac{\epsilon_J}{\sqrt{d}L_0}$ and $\alpha = \frac{\epsilon_J^{1.5}} {4d^{1.5}L_0^2\sqrt{K}}$, we have that $\alpha_g = \frac{\epsilon_J}{2 d K} \leq \frac{1}{2}$ when $\epsilon_J \leq d$ and $K \geq 1$. Substituting the bound on $\alpha_g$ into \eqref{eqn:Thm3.2_9}, we obtain that
	\begin{align}
		\label{eqn:Thm3.2_10}
		\sum_{k=0}^{K-1} \Exp[ \|g_\delta(\theta_{k})\|^2 ] \leq & \; 2 \Exp[ \|g_\delta(\theta_{0})\|^2 ] + 64d^2d_i L_0^2 \frac{\alpha^2}{\delta^4} E_\mu^2K \nonumber \\
		& + 32(d+4)^2 L_0^2 K + \frac{16d\sigma^2}{\delta^2} K.
	\end{align}
	Moreover, substituting the bound in \eqref{eqn:Thm3.2_10} into the bound in \eqref{eqn:Thm3.2_8}, we get that
	\begin{align}
	\label{eqn:Thm3.2_11}
	& \sum_{k=0}^{K-1}\Exp[ \| \nabla J_\delta (\theta_k) \|^2] \leq \frac{2}{\alpha} (\Exp[ J_\delta(\theta_{K}) - J_\delta(\theta_{0})]) \nonumber \\
	& + 4\sqrt{d}L_0 \frac{\alpha}{\delta} \Exp[\|g_\delta(\theta_0)\|^2] + 64(d+4)^{2.5}L_0^3\frac{\alpha}{\delta} K  & \nonumber \\
	& +  32 d^{1.5}L_0 \sigma^2 \frac{\alpha}{\delta^3} K + 128d^{2.5}d_i L_0^3 \frac{\alpha^3}{\delta^5} E_\mu^2 K \nonumber \\
	& + 8d^{1.5}d_iL_0 \frac{\alpha}{\delta^3} E_\mu^2 K + \frac{d d_i}{\delta^2} E_\mu^2 K.
	\end{align}
	Recalling that $E_\mu^2 = \frac{\epsilon \delta^2}{2 dd_i}$ and substituting the selected values for $\delta = \frac{\epsilon_J}{\sqrt{d}L_0}$ and $\alpha = \frac{\epsilon_J^{1.5}} {4d^{1.5}L_0^2\sqrt{K}}$ into \eqref{eqn:Thm3.2_11}, we obtain that
	\begin{align}
		\label{eqn:Thm3.2_12}
		& \sum_{k=0}^{K-1}\Exp[ \| \nabla J_\delta (\theta_k) \|^2] \leq \frac{8d^{1.5}L_0^2}{\epsilon_J^{1.5}} \Exp[ J_\delta^\ast - J_\delta(\theta_{0})] \sqrt{K} \nonumber \\
		& + \frac{\epsilon_J^{0.5}}{\sqrt{dK}} \Exp[\|g_\delta(\theta_0)\|^2] + \frac{\epsilon \epsilon_J^{1.5}}{d^{1.5} \sqrt{K}} + 16 \frac{(d+4)^2}{d} L_0^2 \epsilon_J^{0.5} \sqrt{K} & \nonumber \\
		& + \frac{8d^{1.5}L_0^2 \sigma^2}{\epsilon_J^{1.5}} \sqrt{K} + \frac{\epsilon \epsilon_J^{0.5}}{\sqrt{d}} \sqrt{K} + \frac{\epsilon}{2} K,
	\end{align}	
	where $J_\delta^\ast \geq J_\delta(\theta)$ for all $\theta \in \mathbb{R}^d$. The upper bound on $J_\delta(\theta)$ exists due to Assumption~\ref{asm:BoundVal}. Dividing both sides of \eqref{eqn:Thm3.2_12} by $K$, we achieve the bound in Theorem~\ref{thm:DZO}. 
\end{proof}

\section{Proof of Lemma~\ref{lem:AvgMu}}
{\bf Lemma 3.3.} Let Assumption~\ref{asm:Graph} hold. Then, running Algorithm~\ref{alg:DZO} with $\mathtt{DoTracking} = \mathtt{True}$, we have that $\bar{\mu}^k(m) = J(\theta_{k} + \delta u_k, \xi_k) = \frac{1}{N}\sum_{i = 1}^{N} J_i(\theta_{k} + \delta u_k, \xi_k)$, for all $m = 1, 2, \dots, N_c$ and all $k$.
\begin{proof}
	According to \eqref{eqn:Lem3.1_1}, we have that
	\begin{align}
	\label{eqn:Lem3.3_1}
		\bar{\mu}^k(m) = \frac{1}{N} \mathbf{1}^T & \vec{\mu}^k(m)= \frac{1}{N} \mathbf{1}^T  \vec{\mu}^k(m-1) = \dots  \nonumber \\
		& = \frac{1}{N} \mathbf{1}^T  \vec{\mu}^k(0), \text{ for all }\; m, k \geq 0.
	\end{align}
	Next, we show that $\bar{\mu}^{k}(N_c) = \bar{\mu}^{k}(0) = \frac{1}{N}\sum_{i = 1}^{N} J_i(\theta_{k} + \delta u_k, \xi_k)$ for all $k$. We use mathematical induction to construct the proof. Specifically, suppose that $\bar{\mu}^{k-1}(N_c) = \bar{\mu}^{k-1}(0)  = \frac{1}{N}\sum_{i = 1}^{N} J_i(\theta_{k-1} + \delta u_{k-1}, \xi_{k-1})$ holds. Then, according to line 9 in Algorithm~\ref{alg:DZO}, we have that
	\begin{align}
		&\frac{1}{N} \mathbf{1}^T \vec{\mu}^k(N_c)  =  \frac{1}{N} \mathbf{1}^T \vec{\mu}^k(0)\nonumber \\
		&  = \frac{1}{N} \mathbf{1}^T \big( \mu^{k-1}(N_c) + \vec{J}(\theta_k + \delta u_k, \xi_k) \nonumber \\
		& \quad \quad \quad \quad \quad \quad \quad \quad \quad \quad \quad  \quad \quad \quad  - \vec{J}(\theta_{k-1} + \delta u_{k-1}, \xi_{k-1}) \big) \nonumber \\
		&= \frac{1}{N}\sum_{i = 1}^{N} J_i(\theta_{k-1} + \delta u_{k-1}, \xi_{k-1}) + \frac{1}{N}\sum_{i = 1}^{N} J_i(\theta_{k} + \delta u_{k}, \xi_{k}) \nonumber \\
		& \quad\quad \quad \quad \quad \quad \quad \quad \quad \quad \quad - \frac{1}{N}\sum_{i = 1}^{N} J_i(\theta_{k-1} + \delta u_{k-1}, \xi_{k-1}) \nonumber \\
		& = \frac{1}{N}\sum_{i = 1}^{N} J_i(\theta_{k} + \delta u_{k}, \xi_{k}), \nonumber
	\end{align}
	where $\vec{J}(\theta_k + \delta u_k, \xi_k) = [\dots, J_i(\theta_{k} + \delta u_k, \xi_k), \dots]^T$. The second equality above is due to the induction hypothesis. 
	We have that the induction hypothesis is satisfied for $\bar{\mu}^{0}(N_c)$, according to line 7 in Algorithm~\ref{alg:DZO}. Therefore, we have shown that $\frac{1}{N} \mathbf{1}^T  \vec{\mu}^k(0) = \frac{1}{N}\sum_{i = 1}^{N} J_i(\theta_{k} + \delta u_k, \xi_k)$ for all $k$. And due to \eqref{eqn:Lem3.3_1}, we have shown that $\bar{\mu}^k(m) = J(\theta_{k} + \delta u_k, \xi_k) = \frac{1}{N}\sum_{i = 1}^{N} J_i(\theta_{k} + \delta u_k, \xi_k)$, for all $m = 1, 2, \dots, N_c$ and all $k$. The proof is complete.
\end{proof}

\section{Proof of Lemma~\ref{lem:ConsensusErr_2}}
{\bf Lemma 3.4.} Let Assumptions~\ref{asm:Lipschitz}, \ref{asm:UnbiasedEval}, \ref{asm:Graph} hold and define  $E_\mu^k = \|\vec{\mu}^k(N_c) - \bar{\mu}^k(N_c) \mathbf{1}\|$. Then, running Algorithm~\ref{alg:DZO} with $\mathtt{DoTracking} = \mathtt{True}$, we have that
\begin{align}
	\mathbb{E}\big[ (E_\mu^k)^2 \big] \leq & \bigg( 2\mathbb{E}\big[ (E_\mu^{k-1})^2 \big] + 32dL_0^2\frac{\alpha^2}{\delta^2}\mathbb{E}\big[ (E_\mu^{k-1})^2 \|u_{k-1}\|^2 \big] \nonumber \\
	& + 32d^2L_0^2\frac{\alpha^2}{\delta^2} \mathbb{E}\big[ (E_\mu^{k-2})^2 \big] \bigg) \rho_W^{2N_c} & \nonumber \\
	& + 16NL_0^2 \alpha^2 \mathbb{E}\big[ \| \tilde{\nabla} J(\theta_{k-1})\|^2 \big] \rho_W^{2N_c} \nonumber \\
	& + 32Nd L_0^2 \delta^2 \rho_W^{2N_c} + 16N\sigma^2 \rho_W^{2N_c} . &
\end{align}
\begin{proof}
	To simplify notations, in what follows, we denote $\vec{\mu}^k(N_c)$ and $\bar{\mu}^k(N_c)$ by $\vec{\mu}^k$ and $\bar{\mu}^k$ respectively. According to lines 9 and 12 in Algorithm~\ref{alg:DZO}, we have that 
	\begin{align}
		& \mathbb{E}\big[ (E_\mu^k)^2 \big] = \Exp [ \|(I - \frac{1}{N}\mathbf{1}\mathbf{1}^T) \vec{\mu}^k\|^2 ] \nonumber \\
		&= \Exp [ \|(I - \frac{1}{N}\mathbf{1}\mathbf{1}^T) W^{N_c} \big(\mu^{k-1} + \vec{J}(\theta_k + \delta u_k, \xi_k) \nonumber \\
		& \quad  \quad  - \vec{J}(\theta_{k-1} + \delta u_{k-1}, \xi_{k-1}) \big) \|^2 ] \nonumber \\
		& = \Exp [ \|( W^{N_c} - \frac{1}{N}\mathbf{1}\mathbf{1}^T) \big(\mu^{k-1} + \vec{J}(\theta_k + \delta u_k, \xi_k) \nonumber \\
		&  \quad  \quad - \vec{J}(\theta_{k-1} + \delta u_{k-1}, \xi_{k-1}) \big) \|^2 ]. \nonumber
	\end{align}
	Since $( W^{N_c} - \frac{1}{N}\mathbf{1}\mathbf{1}^T) \bar{\mu}^{k-1} \mathbf{1} = 0$, we obtain that
	\begin{align}
		\label{eqn:Lem3.4_2}
		& \mathbb{E}\big[ (E_\mu^k)^2 \big] = \Exp [ \|(W^{N_c} - \frac{1}{N}\mathbf{1}\mathbf{1}^T)\big(\mu^{k-1} - \bar{\mu}^{k-1} \mathbf{1}  \nonumber \\
		& + \vec{J}(\theta_k + \delta u_k, \xi_k) - \vec{J}(\theta_{k-1} + \delta u_{k-1}, \xi_{k-1}) \big) \|^2 ] \nonumber \\
		& \leq 2 \Exp[ \|W^{N_c} - \frac{1}{N}\mathbf{1}\mathbf{1}^T\|^2 (E_\mu^{k-1})^2 ] + 2 \Exp[ \|W^{N_c}  \\
		&  - \frac{1}{N}\mathbf{1}\mathbf{1}^T\|^2 \|\vec{J}(\theta_k + \delta u_k, \xi_k) - \vec{J}(\theta_{k-1} + \delta u_{k-1}, \xi_{k-1})\|^2 ]. \nonumber 
	\end{align}
	Moreover, since $(W -  \frac{1}{N}\mathbf{1}\mathbf{1}^T)^2 = W^2 -  \frac{1}{N}\mathbf{1}\mathbf{1}^T$, we get that $(W -  \frac{1}{N}\mathbf{1}\mathbf{1}^T)^3 = (W^2 -  \frac{1}{N}\mathbf{1}\mathbf{1}^T) (W - \frac{1}{N}\mathbf{1}\mathbf{1}^T = W^3 - \frac{1}{N}\mathbf{1}\mathbf{1}^T)$. This is because $W^2$ is also doubly stochastic, which can be shown using the definition. Through induction, we obtain that $(W - \frac{1}{N}\mathbf{1}\mathbf{1}^T)^{N_c} = W^{N_c} - \frac{1}{N}\mathbf{1}\mathbf{1}^T$. Therefore, we have that $\|W^{N_c} - \frac{1}{N}\mathbf{1}\mathbf{1}^T\|^2 = \|(W - \frac{1}{N}\mathbf{1}\mathbf{1}^T)^{N_c}\|^2 \leq \|W - \frac{1}{N}\mathbf{1}\mathbf{1}^T\|^{2N_c} = \rho_W^{2N_c}$. Applying this bound in \eqref{eqn:Lem3.4_2}, we get that
	\begin{align}
		\label{eqn:Lem3.4_3}
		& \mathbb{E}\big[ (E_\mu^k)^2 \big]  \leq 2 \rho_W^{2N_c} \Exp[(E_\mu^{k-1})^2] \\
		& + 2 \rho_W^{2N_c}  \Exp[ \|\vec{J}(\theta_k + \delta u_k, \xi_k) - \vec{J}(\theta_{k-1} + \delta u_{k-1}, \xi_{k-1})\|^2 ]. \nonumber 
	\end{align}
	Following the same procedure used to derive the bound in \eqref{eqn:LemB.3_3}, we get that $|J_i(\theta_k + \delta u_k, \xi_k) - J_i(\theta_{k-1} + \delta u_{k-1}, \xi_{k-1})|^2 \leq 4L_0^2 \|\theta_k - \theta_{k-1}\|^2 + 4 L_0^2 \delta^2 \|u_k - u_{k-1}\|^2 + 8\sigma^2$, for all $i = 1, 2, \dots, N$. Applying this bound to the RHS in \eqref{eqn:Lem3.4_3}, we obtain that 
	\begin{align}
	\label{eqn:Lem3.4_4}
	\mathbb{E}\big[ & (E_\mu^k)^2 \big]  \leq 2 \rho_W^{2N_c} \Exp[(E_\mu^{k-1})^2] + 8NL_0^2 \rho_W^{2N_c}  \Exp[ \|\theta_k - \theta_{k-1}\|^2  ]  \nonumber \\
	& + 8NL_0^2 \delta^2 \rho_W^{2N_c} \Exp[ \|u_k - u_{k-1}\|^2 ] + 16N \sigma^2 \rho_W^{2N_c} .
	\end{align}	
	Moreover, we have that $\Exp[ \|u_k - u_{k-1}\|^2 ] \leq \Exp[ 2\|u_k\|^2 + 2 \|u_{k-1}\|^2 ] \leq 4d$. Substituting the expression of $\theta_k - \theta_{k-1}$ for \eqref{eqn:Thm3.2_3} into \eqref{eqn:Lem3.4_4} and applying the bound in \eqref{eqn:Thm3.2_7}, we obtain that
	\begin{align}
		\label{eqn:Lem3.4_5}
		\mathbb{E}\big[ (E_\mu^k)^2 \big]  \leq &\; 2 \rho_W^{2N_c} \Exp[(E_\mu^{k-1})^2] + 16NL_0^2 \rho_W^{2N_c} \alpha^2 \Exp[ \|g_\delta(\theta_{k-1}) \|^2] \nonumber \\
		& + 32dNL_0^2 \delta^2 \rho_W^{2N_c} + 16N \sigma^2 \rho_W^{2N_c} \nonumber \\ 
		& + 32dL_0^2 \rho_W^{2N_c} \frac{\alpha^2}{\delta^2} \Exp[ (E_\mu^{k-1})^2 \|u_{k-1}\|^2 ]  \nonumber \\
		& + 32d^2L_0^2 \rho_W^{2N_c} \frac{\alpha^2}{\delta^2} \Exp[ (E_\mu^{k-2})^2 ] \nonumber.
	\end{align}	
	The proof is complete.
\end{proof}

\section{Proof of Theorem~\ref{thm:DZO_ValueTrack}}
First, we present a lemma characterizing the bound on $\Exp[ (E_\mu^k)^2 \|u_k\|^2]$.
\begin{lemma}
	\label{lem:ConsensusErr_3}
	Let Assumptions~\ref{asm:Lipschitz}, \ref{asm:UnbiasedEval}, \ref{asm:Graph} hold. Then, for all $k \geq 1$, we have that 
	\begin{equation}
		\label{eqn:LemE.1}
		\begin{split}
		& \mathbb{E}\big[ (E_\mu^k)^2 \|u_k\|^2 \big]  \nonumber \\
		& \leq 2 d \rho_W^{2N_c} \Exp[(E_\mu^{k-1})^2] + 16d NL_0^2 \rho_W^{2N_c} \alpha^2 \Exp[ \|g_\delta(\theta_{k-1}) \|^2]  \\ 
		&\quad  + 32d^2L_0^2 \rho_W^{2N_c} \frac{\alpha^2}{\delta^2} \Exp[ (E_\mu^{k-1})^2 \|u_{k-1}\|^2 ]  \nonumber \\
		& \quad + 32d^3L_0^2 \rho_W^{2N_c} \frac{\alpha^2}{\delta^2} \Exp[ (E_\mu^{k-2})^2 ] \\
		& \quad + 32(d+4)^2NL_0^2 \delta^2 \rho_W^{2N_c} + 16dN \sigma^2 \rho_W^{2N_c}.
		\end{split}
	\end{equation}	
	
\end{lemma}
\begin{proof}
	According to the bound on $(E_\mu^k)^2$ derived in \eqref{eqn:Lem3.4_2}, we have that
	\begin{align}
	\label{eqn:LemE.1_1}
		& \mathbb{E}\big[ (E_\mu^k)^2 \|u_k\|^2 \big]  \leq 2 \rho_W^{2N_c} \Exp[(E_\mu^{k-1})^2 \|u_k\|^2]  \\
		&  + 2 \rho_W^{2N_c}  \Exp[ \|\vec{J}(\theta_k + \delta u_k, \xi_k) - \vec{J}(\theta_{k-1} + \delta u_{k-1}, \xi_{k-1})\|^2 \|u_k\|^2 ]. \nonumber
	\end{align}
	Since $E_\mu^{k-1}$ is independent of $u_k$, we have that $\mathbb{E}\big[ (E_\mu^{k-1})^2 \|u_k\|^2 \big] = \Exp[ (E_\mu^{k-1})^2 ] \Exp[\|u_k\|^2] = d\Exp[ (E_\mu^{k-1})^2 ] $. Following the same procedure used to derive the bound in \eqref{eqn:LemB.3_3}, we get that $|J_i(\theta_k + \delta u_k, \xi_k) - J_i(\theta_{k-1} + \delta u_{k-1}, \xi_{k-1})|^2 \leq 4L_0^2 \|\theta_k - \theta_{k-1}\|^2 + 4 L_0^2 \delta^2 \|u_k - u_{k-1}\|^2 + 8\sigma^2$, for all $i = 1, 2, \dots, N$. Applying this bound to the RHS in \eqref{eqn:LemE.1_1}, we obtain that 
		\begin{align}
	\label{eqn:LemE.1_2}
	\mathbb{E}\big[ (& E_\mu^k)^2 \|u_k\|^2 \big]  \leq 2d \rho_W^{2N_c} \Exp[(E_\mu^{k-1})^2] \\
	&  + 8NL_0^2 \rho_W^{2N_c}  \Exp[ \|\theta_k - \theta_{k-1}\|^2 \|u_k\|^2 ]  \nonumber \\
	& + 8NL_0^2 \delta^2 \rho_W^{2N_c} \Exp[ \|u_k - u_{k-1}\|^2 \|u_k\|^2 ] + 16 dN \sigma^2 \rho_W^{2N_c}. \nonumber
	\end{align}	
	Moreover, we have that $\Exp[ \|u_k - u_{k-1}\|^2 \|u_k\|^2 ] \leq \Exp[ 2\|u_k\|^4 + 2 \|u_{k-1}\|^2\|u_k\|^2 ] \leq 4(d+4)^2$. Substituting the expression of $\theta_k - \theta_{k-1}$ in \eqref{eqn:Thm3.2_3} into \eqref{eqn:LemE.1_2} and applying the bound \eqref{eqn:Thm3.2_7}, we obtain that
	\begin{align}
	\label{eqn:LemE.1_3}
	& \mathbb{E}\big[ (E_\mu^k)^2 \|u_k\|^2 \big]   \nonumber \\
	& \leq 2 d \rho_W^{2N_c} \Exp[(E_\mu^{k-1})^2] + 16d NL_0^2 \rho_W^{2N_c} \alpha^2 \Exp[ \|g_\delta(\theta_{k-1}) \|^2] \nonumber \\ 
	& + 32d^2L_0^2 \rho_W^{2N_c} \frac{\alpha^2}{\delta^2} \Exp[ (E_\mu^{k-1})^2 \|u_{k-1}\|^2 ]  \nonumber \\
	& + 32d^3L_0^2 \rho_W^{2N_c} \frac{\alpha^2}{\delta^2} \Exp[ (E_\mu^{k-2})^2 ] \nonumber \\
	& + 32(d+4)^2NL_0^2 \delta^2 \rho_W^{2N_c} + 16dN \sigma^2 \rho_W^{2N_c} 
	\end{align}
	The proof is complete.
\end{proof}

Now, we are ready to present the proof for Theorem~\ref{thm:DZO_ValueTrack}.

{\bf Theorem 3.5. (Learning Rate of Algorithm~\ref{alg:DZO} with Value Tracking)}
Let Assumptions~\ref{asm:Lipschitz},\ref{asm:UnbiasedEval}, \ref{asm:Graph} hold and define	$\delta = \frac{\epsilon_J}{\sqrt{d}L_0}$, $ \alpha = \frac{\epsilon_J^{1.5}} {4d^{1.5}L_0^2\sqrt{K}}$, and $N_c \geq \max\big(\log(\frac{1}{2\sqrt{2}})/\log(\rho_W)$, $N_c \geq \log(\sqrt{ \frac{\epsilon}{ 4\big( G^2 \epsilon_J + 32 (d+4)^2d L_0^2  + 16 d^3L_0^2\sigma^2/\epsilon_J^2 \big)}}) / \log(\rho_W) \big)$
where $G^2 = \max \bigg( \mathbb{E}\big[ \| \tilde{\nabla} J(\theta_0) \|^2\big], \frac{2 \epsilon_J \epsilon }{d K} + 32L_0^2(d+4)^2$ $+ 16d^2L_0^2 \frac{\sigma^2}{\epsilon_J^2} \bigg)$. Then, running Algorithm~\ref{alg:DZO} with $\mathtt{DoTracking} = \mathtt{True}$, we have that $\frac{1}{K} \sum_{k = 0}^{K-1} \mathbb{E}[ \| \nabla J_\delta(\theta_k) \|^2 ] \leq \mathcal{O}( d^{1.5} \epsilon_J^{-1.5} K^{-0.5} ) + \frac{\epsilon}{2}$.
\begin{proof}
	First, we show that for all $k \geq 0$, we have that $\Exp[ \|g_\delta(\theta_k)\|^2 ] \leq G^2$, $\Exp[ (E_\mu^k)^2 ] \leq E_\mu^2$ and $\Exp[ (E_\mu^k)^2 \|u_k\|^2 ] \leq d E_\mu^2$ when we let  $\delta = \frac{\epsilon_J}{\sqrt{d}L_0}$, $ \alpha = \frac{\epsilon_J^{1.5}} {4d^{1.5}L_0^2\sqrt{K}}$, and $N_c \geq \max\big(\log(\frac{1}{2\sqrt{2}})/\log(\rho_W)$, $\log(\sqrt{ \frac{\epsilon}{ 2G^2 \epsilon_J + 64 (d+4)^2d L_0^2  + 32d^3L_0^2\sigma^2/\epsilon_J^2}}) / \log(\rho_W) \big)$, where $E_\mu^2 = \frac{\epsilon \delta^2}{2 dd_i}$. To prove this, we use mathematical induction. Specifically, suppose we have that $\Exp[ (E_\mu^{k-1})^2 ] \leq E_\mu^2$, $\Exp[ (E_\mu^{k-2})^2 ] \leq E_\mu^2$, $\Exp[ (E_\mu^{k-1})^2 \|u_{k-1}\|^2 ] \leq dE_\mu^2$ and $\Exp[ \|g_\delta(\theta_{k-1})\|^2 ] \leq G^2$. Then, according to Lemma~\ref{lem:SecondMoment}, we have that
	\begin{align}
		\label{eqn:Thm3.5_1}
		\Exp[ \|g_\delta(\theta_{k})\|^2] \leq &\; 8dL_0^2\frac{\alpha^2}{\delta^2} G^2 + 32d^2d_i L_0^2 \frac{\alpha^2}{\delta^4} E_\mu^2 \nonumber \\
		& + 16(d+4)^2 L_0^2 + \frac{8d\sigma^2}{\delta^2}.
	\end{align}
	Substituting the selected values for $\delta$, $\alpha$ and the constant $E_\mu^2$ in \eqref{eqn:Thm3.5_1}, we get that
	\begin{align}
		\label{eqn:Thm3.5_2}
		\Exp[ \|g_\delta(\theta_{k})\|^2] & \leq\frac{\epsilon_J}{2dK}G^2 + \frac{\epsilon \epsilon_J}{dK} + 16L_0^2(d+4)^2 + 8d^2L_0^2\frac{\sigma^2}{\epsilon_J^2} , \nonumber \\
		& \leq G^2
	\end{align}
	where the second inequality holds because $\frac{\epsilon_J}{2dK}G^2 \leq \frac{1}{2}G^2$ when $\frac{\epsilon_J}{dK} \leq 1$. In addition, we have that $\frac{\epsilon \epsilon_J}{dK} + 16L_0^2(d+4)^2 + 8d^2L_0^2\frac{\sigma^2}{\epsilon_J^2} \leq \frac{1}{2}G^2$ due to the choice of $G^2$ in Theorem~\ref{thm:DZO_ValueTrack}.
	Furthermore, according to Lemma~\ref{lem:ConsensusErr_2}, we have that
	\begin{align}
		\label{eqn:Thm3.5_3}
		\mathbb{E}\big[ (E_\mu^k)^2 \big] \leq &\; \big( 2 + 64d^2L_0^2\frac{\alpha^2}{\delta^2}\big) \rho_W^{2N_c} E_\mu^2 + 16NL_0^2 G^2 \alpha^2  \rho_W^{2N_c} \nonumber \\
		& + 32Nd L_0^2 \delta^2 \rho_W^{2N_c} + 16N\sigma^2 \rho_W^{2N_c} .
	\end{align}
	Substituting the selected values for $\delta$, $\alpha$ and $E_\mu^2$ into \eqref{eqn:Thm3.5_3}, we get that
	\begin{align}
		\label{eqn:Thm3.5_4}
		& \mathbb{E}\big[ (E_\mu^k)^2 \big] \nonumber \\
		& \leq (2 + 4 \frac{\epsilon_J}{K}) \rho_W^{2N_c} E_\mu^2 + (\frac{NG^2\epsilon_J^3}{d^3L_0^2K} + 32N\epsilon_J^2 + 16N\sigma^2) \rho_W^{2N_C} \nonumber \\
		&  \leq E_\mu^2.
	\end{align}
	The second inequality is because when $\epsilon_J/K \leq \frac{1}{2}$ and $N_c \geq \log(\frac{1}{2\sqrt{2}})/\log(\rho_W)$, we have that $(2 + 4 \frac{\epsilon_J}{K}) \rho_W^{2N_c} E_\mu^2 \leq \frac{1}{2} E_\mu^2$. In addition, when $N_c \geq \log(\sqrt{ \frac{\epsilon}{ 4\big( G^2 \epsilon_J + 32 (d+4)^2d L_0^2  + 16 d^3L_0^2\sigma^2/\epsilon_J^2 \big)}}) / \log(\rho_W)$, we get that $(\frac{NG^2\epsilon_J^3}{d^3L_0^2K} + 32N\epsilon_J^2 + 16N\sigma^2) \rho_W^{2N_C} \leq \frac{1}{2} E_\mu^2$. 
	
	Next, according to Lemma~\ref{lem:ConsensusErr_3}, we have that
	\begin{align}
		\label{eqn:Thm3.5_5}
		& \Exp \big[ (E_\mu^k)^2 \|u_k\|^2 \big]  \leq \big( 2 + 64d^2L_0^2\frac{\alpha^2}{\delta^2}\big) \rho_W^{2N_c} dE_\mu^2  \\
		&+ 16dNL_0^2 G^2 \alpha^2  \rho_W^{2N_c} + 32N(d+4)^2L_0^2 \delta^2 \rho_W^{2N_c} + 16dN\sigma^2 \rho_W^{2N_c}. \nonumber 
	\end{align}
	Substituting the selected values for $\delta$, $\alpha$ and $E_\mu^2$ into \eqref{eqn:Thm3.5_5}, we get that $\mathbb{E}\big[ (E_\mu^k)^2 \|u_k\|^2 \big] \leq  (2 + 4 \frac{\epsilon_J}{K}) \rho_W^{2N_c} dE_\mu^2 
	+ (\frac{NG^2\epsilon_J^3}{d^2L_0^2K}  + 32N \frac{(d+4)^2}{d}\epsilon_J^2 + 16dN\sigma^2) \rho_W^{2N_C}  \leq d E_\mu^2$.
The second inequality holds for similar reasons as those used to obtain \eqref{eqn:Thm3.5_4}. To complete the induction argument, we simply need to verify the induction hypothesis when $k = 1$. It is straightforward to see that $\Exp[ \|g_\delta(\theta_{0})\|^2 ] \leq G^2$ due to the definition of $G^2$. In addition, due to the initialization step $\mu^{-1}(N_c) = 0$ in line 1 in Algorithm~\ref{alg:DZO}, we have that $\Exp[ (E_\mu^{-1})^2 ] \leq E_\mu^2$. To satisfy the conditions $\Exp[ (E_\mu^{0})^2 ] \leq E_\mu^2$ and $\Exp[ (E_\mu^{0})^2 \|u_{0}\|^2 ] \leq dE_\mu^2$, it is sufficient to run many enough consensus steps only at the first iteration of Algorithm~\ref{alg:DZO}, according to Lemma~\ref{lem:ConsensusErr_1}. To summarize, the induction hypothesis is satisfied at the first iteration of Algorithm~\ref{alg:DZO} and we have shown that for all $k \geq 0$, we have that $\Exp[ \|g_\delta(\theta_k)\|^2 ] \leq G^2$, $\Exp[ (E_\mu^k)^2 ] \leq E_\mu^2$ and $\Exp[ (E_\mu^k)^2 \|u_k\|^2 ] \leq d E_\mu^2$ under the choice of parameters specified in Theorem~\ref{thm:DZO_ValueTrack}.

Finally, using the uniform bounds $\Exp[ (E_\mu^k)^2 ] \leq E_\mu^2$ and $\Exp[ (E_\mu^k)^2 \|u_k\|^2 ] \leq d E_\mu^2$, we can follow the same procedure as in the proof of Theorem~\ref{thm:DZO} and obtain the following optimality bound
	\begin{align}
& \sum_{k=0}^{K-1}\Exp[ \| \nabla J_\delta (\theta_k) \|^2] \leq \frac{8d^{1.5}L_0^2}{\epsilon_J^{1.5}} \Exp[ J_\delta^\ast - J_\delta(\theta_{0})] \sqrt{K}  \nonumber \\
& + \frac{\epsilon_J^{0.5}}{\sqrt{dK}} \Exp[\|g_\delta(\theta_0)\|^2] + \frac{\epsilon \epsilon_J^{1.5}}{d^{1.5} \sqrt{K}} + 16 \frac{(d+4)^2}{d} L_0^2 \epsilon_J^{0.5} \sqrt{K}& \nonumber \\
& + \frac{8d^{1.5}L_0^2 \sigma^2}{\epsilon_J^{1.5}} \sqrt{K} + \frac{\epsilon \epsilon_J^{0.5}}{\sqrt{d}} \sqrt{K} + \frac{\epsilon}{2} K.
\end{align}	
Dividing both sides by $K$ completes the proof.
\end{proof}

\begin{IEEEbiography}[{\includegraphics[width=1in,height=1.25in,clip,keepaspectratio]{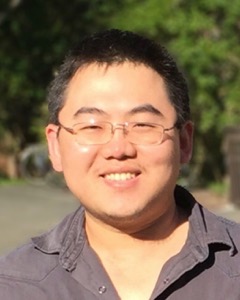}}]{Yan Zhang (S’16)}
	received his bachelor’s degree in mechanical engineering from Tsinghua University, Beijing, China in 2014, and his master's degree in mechanical engineering from Duke University, Durham, NC in 2016. He received his doctoral degree in mechanical engineering at Duke University, Durham, NC in 2021.
	
	His research interests include distributed optimization and distributed reinforcement learning algorithms.
\end{IEEEbiography}

\begin{IEEEbiography}
[{\includegraphics[width=1in,height=1.25in,clip,keepaspectratio]{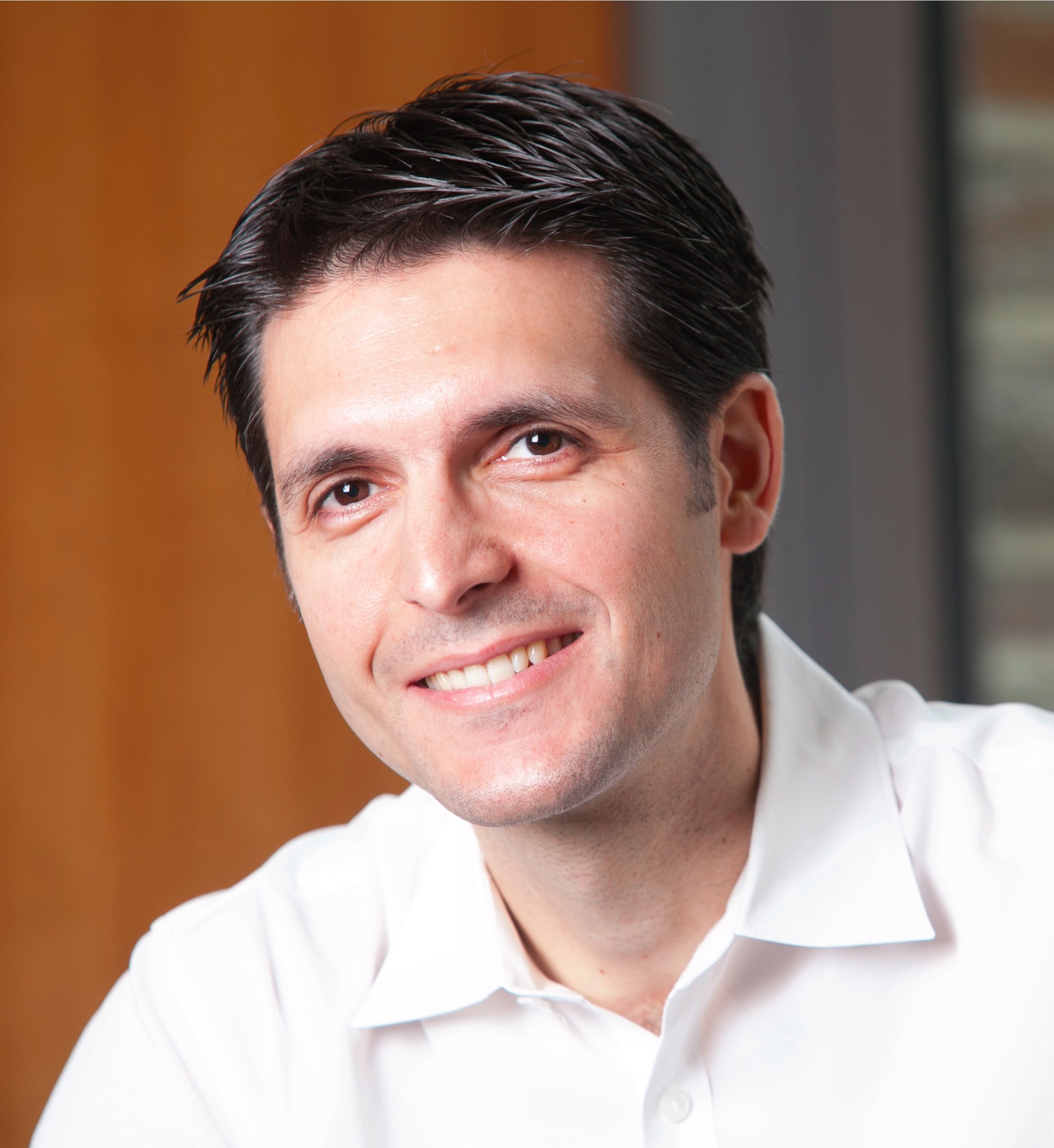}}]{Michael M. Zavlanos}  (S'05M'09SM'19) received the Diploma in mechanical engineering from the National Technical University of Athens, Greece, in 2002, and the M.S.E. and Ph.D. degrees in electrical and systems engineering from the University of Pennsylvania, Philadelphia, PA, in 2005 and 2008, respectively. 

He is currently the Yoh Family Professor of the Department of Mechanical Engineering and Materials Science at Duke University, Durham, NC. He also holds a secondary appointment in the Department of Electrical and Computer Engineering and the Department of Computer Science. His research focuses on control theory, optimization, and learning with applications in robotics and autonomous systems, cyber-physical systems, and healthcare/medicine. Dr. Zavlanos is a recipient of various awards including the 2014 ONR YIP Award and the 2011 NSF CAREER Award.
\end{IEEEbiography}

\end{document}